%% file: arxiv.tex
\documentclass[12pt]{demo} 

\title[Optimization in Machine Learning: A Distribution Space Approach]{Optimization in Machine Learning: A Distribution Space Approach}
\usepackage{times}

\author{%
\Name{Yongqiang Cai} \Email{matcyon@nus.edu.sg}\\
 \addr Department of Mathematics,
 National University of Singapore,
 Singapore
 \AND
 \Name{Qianxiao Li} \Email{qianxiao@nus.edu.sg}\\
 \addr Department of Mathematics,
 National University of Singapore,
 Singapore%
 \AND
 \Name{Zuowei Shen} \Email{matzuows@nus.edu.sg}\\
 \addr Department of Mathematics,
 National University of Singapore,
 Singapore%
}

\usepackage{hyperref}       
\usepackage{url}            
\usepackage{booktabs}       
\usepackage{amsfonts}       
\usepackage{nicefrac}       
\usepackage{microtype}      

\usepackage{hyperref}
\usepackage{url}
\usepackage{color}
\usepackage{wrapfig}
\usepackage{algorithm}
\usepackage{algorithmic}

\newcommand{\Ucal}{\mathcal{U}}
\newcommand{\R}{\mathbb{R}}

\DeclareMathOperator{\conv}{conv}
\DeclareMathOperator{\relu}{ReLU}
\DeclareMathOperator{\argmin}{arg\,min}

\begin{document}

\maketitle

\begin{abstract}%
    We present the viewpoint that optimization problems encountered in machine learning can often be interpreted as minimizing a convex functional over a function space, but with a non-convex constraint set introduced by model parameterization. This observation allows us to repose such problems via a suitable relaxation as convex optimization problems in the space of distributions over the training parameters. We derive some simple relationships between the distribution-space problem and the original problem, e.g. a distribution-space solution is at least as good as a solution in the original space. Moreover, we develop a numerical algorithm based on mixture distributions to perform approximate optimization directly in distribution space. Consistency of this approximation is established and the numerical efficacy of the proposed algorithm is illustrated on simple examples. In both theory and practice, this formulation provides an alternative approach to large-scale optimization in machine learning.

\end{abstract}


\input{tex/main.tex}


\bibliography{rfs}

\newpage
\appendix

\input{tex/appendix.tex}

\end{document}

%% file: tex/main.tex
%
\section{Introduction}
\label{sec:intro}

Many optimization problems arising in machine learning applications are of the form
\begin{align}\label{eq:main_objective}
    \inf_{u \in \Ucal} J[u],
\end{align}
where $J:\Ucal \rightarrow \R$ is a convex loss functional and $\Ucal$ is the model hypothesis space. For example, in least squares regression over a Euclidean domain $\Omega$ with ground truth $f: \Omega \rightarrow \R$ and data distribution $\nu$, the loss functional takes the form $J[u] = \int_\Omega | u(x) - f(x) |^2 d\nu(x)$. In general, $\Ucal$ is an infinite dimensional space, e.g. the space of continuous functions. Hence, directly minimizing $J$ over $\Ucal$ is computationally difficult and we often resort to some parametric form of~\eqref{eq:main_objective}.

Concretely, we consider restricting~\eqref{eq:main_objective} to a parameterized hypothesis space  $\Ucal_W := \{ u_w: w\in W \} \subset \Ucal$, where $W \subset \R^d$ is a finite-dimensional Euclidean space of trainable parameters. Then, we obtain the parameterized optimization problem
\begin{align}\label{eq:para_objective}
    \inf_{u \in \Ucal_W} J[u]
    \equiv
    \inf_{w \in W} J[u_w]
    \equiv
    \inf_{w \in W} L(w)   \qquad   (L(w) := J[u_w]).
\end{align}
In fact, the right hand side of~\eqref{eq:para_objective} is the more familiar form encountered in practice, where the additional structure of $J,u_w$ is not explicitly represented, and absorbed into a common non-convex loss function $L$. For example, in the case of linear regression, we set $u_{v,b}(x) = v^T x + b$ and so we obtain the supervised learning problem $\min_{v,b} L(v,b) \equiv \int_{\Omega} | v^T x + b - f(x) |^2 d\nu(x)$. However, in this paper we will exploit the aforementioned additional structure.

One central challenge from the optimization viewpoint is that, although~\eqref{eq:para_objective} is finite-dimensional, it is in general non-convex in $w$ unless $u_w$ is linear in $w$. Consequently, if one applies (stochastic) gradient-based solution methods, the solution quality not only depends on the initial condition~\citep{Glorot2010Understanding} but also on a host of complex interactions between the optimization algorithm and the loss landscape in $W$ space~\citep{Jain2017Non,dauphin2014identifying,wu2017towards}.

This motivates a natural question: can we repose~\eqref{eq:main_objective} in a computationally tractable manner while retaining its highly desirable convexity characteristics? In this work, we introduce a formulation where this can be achieved. The main idea is going into the dual space of $w \mapsto u_w$, which is the space of distributions on $W$. In this distribution space, a convex problem can be defined that can be reduced to a finite-dimensional problem via approximation by mixture distributions, leading to implementable algorithms.

\section{Related work}
\label{sec:related_work}

In this section, we discuss some related work and put our paper in perspective of the relevant literature.
First, one primary motivation for this work is the recent advancements on optimization in over-parameterized settings~\citep{Du2018Gradient,Li2018Learning,Chizat2018Global,allen2018convergence,arora2018optimization,ma2017power,bassily2018exponential,oymak2018overparameterized,Martin2018Implicit,martin2019traditional}, typically applied to understand various interesting phenomenon arising from practical deep learning. For example, it is found empirically that although deep learning models are highly over-parameterized, they can generalize to unseen data~\citep{Zhang2016Understanding}. Among many, one suggested explanation is that deep neural networks have implicit regularization properties, which can be quantified by studying empirical distributions of the trained weights~\citep{Martin2018Implicit}.
Such distribution space ideas also underlie recent work on mean-field limits of deep neural networks~\citep{Chizat2018Global, Mei2018mean, Rotskoff2018Parameters, Sirignano2018Mean},
in which it is argued that trained weights in large neural networks behave like they are sampled from a distribution, and this fact can be used to modify and improve training algorithms~\citep{Xiao2018Dynamical}.

Although motivated by these empirical and theoretical findings, our work concerns a more basic problem: from the outset, can we formulate the optimization of non-convex functions directly in distribution space? This highlights the major difference in our work, in that we are not concerned with ``effective behavior'' of models trained with traditional methods in large system limits. Rather, we study the possibility of formulating the training problem directly on the space of distributions over a fixed parameterized model space. In a similar vein, this also sets the current approach apart from previous studies on convexifying neural networks~\citep{Bengio2006Convex}, where the convex problem is defined on the final-layer parameters, instead of the distribution space over parameters. Furthermore, the results here do not only apply neural networks as they do not depend on their specific structure other than the fact that the optimization problem takes the form~\eqref{eq:para_objective}.
The current line of investigation is also related to the Barron space framework proposed in~\citep{Ma2019Barron}, where the authors studied approximation and generalization properties of neural networks by exploiting the distribution-space viewpoint. In contrast, we focus on the optimization aspect and algorithm development.

On the algorithmic side, there is a line of papers showing that gradient descent training of very wide neural networks can be viewed as a Wasserstein gradient flow in the space of distributions over the trainable parameters~\citep{Chizat2018Global, Mei2018mean, Rotskoff2018Parameters, Sirignano2018Mean, Wei2018margin}. The distribution-space viewpoint is also found in these papers, but these mainly analyze the collective properties of a large number of gradient descent trajectories with random initial conditions. In some sense, this can be thought of as a ``particle method''~\citep{Dean1996Langevin}
for solving an optimization problem in distribution space. In contrast, our method based on mixture distributions works directly in the distribution space, and can be viewed as a ``subspace method''. It is worth reiterating that in our formulation, we are not taking limits of network widths. Rather, we start with a fixed model architecture (which can be big or small, and is not limited to neural networks) and discuss how we can formulate and solve training problems in the space of distributions.

The present approach is also related to the so-called ``random feature models'', which also make use of distributions over parameter in a kernel mapping~\citep{Rahimi2008Random, Rahimi2009Weighted, Sinha2016Learning}.
The key difference is that most of these approaches rely on a fixed distribution, which in the limit of an infinite number of features can span the desired function spaces. However, the present approach actually trains the distribution and does not require the regime of large feature numbers for our analytical results.

Finally, since the distribution space formulation requires averaging models derived from parameter sampling at inference time, this naturally connects the current work with classical literature on model ensembling~\citep{Rokach2010Ensemble, Breiman1996Bagging, Freund1996Experiments}, where modern techniques such as Dropout \citep{Srivastava2014Dropout} and DropConnect \citep{Wan2013Regularization} provide ways of approximately combining exponentially many different neural network models which share some parameters to prevent overfitting.
In fact, in the current framework, the classical ensembling techniques can be regarded as approximating a distribution by finite or countable convex or linear combinations of point masses. In this sense, we can view our proposed methods as effectively lifting the point mass restriction and consider more general classes of basic probability distributions.

\section{Optimization problem in distribution space}
\label{sec:opt_distribution}

In this section, we outline our convex formulation of~\eqref{eq:para_objective}. Although the following results are expected to hold for general Banach spaces, for concreteness, we hereafter take $\Ucal = C^0(\Omega)$ and consider compact sets $\Omega$ and $W$. The simple but crucial observation is that ~\eqref{eq:para_objective} can be rewritten as an optimization problem with a convex loss in function space with a non-convex constraint set:
\begin{align}
    \inf_{u\in \Ucal} J[u] \qquad \text{subject to} \qquad u \in \Ucal_W.
\end{align}
Consequently, the simplest convex relaxation one can adopt is
\begin{align}
    \inf_{u\in \Ucal} J[u] \qquad \text{subject to} \qquad u \in \overline{\conv}(\Ucal_W),
\end{align}
where $\overline\conv(\Ucal_W)$ is the closure of the convex hull of $\Ucal_W$. This has a natural dual-space representation as an optimization problem over distributions
\begin{align}\label{eq:measure_objective}
    \inf_{\mu \in \mathcal{M}}
    F[\mu] :=
    J
    \left[
        \int_{W} u_w d\mu(w)
    \right]
\end{align}
where $\mathcal{M}$ denotes the space of probability measures on $W$. If $\mu$ has a density $\rho$ with respect to the Lebesgue measure, we abuse notation slightly and write $F[\rho] := J \left[\int_{W} u_w \rho(w )dw\right]$. Note that~\eqref{eq:measure_objective} is now a convex problem in $\mu$ since $J$ is convex, the term in its argument is linear in $\mu$ and $\mathcal{M}$ is a convex set. In addition, the functions in $\overline\conv(\mathcal{U}_W)$ can be characterized by $\overline\conv(\mathcal{U}_W) = \{\int_{W} u_w d\mu(w) : \mu \in \mathcal{M}\}$.
In practice, it is often enough to consider probability measures that have a uniformly continuous density function $\rho$. We thus denote by $P$ the space of all such uniformly continuous probability density functions.

In the following, we present some results on the relationships between the distribution space minimization problem \eqref{eq:measure_objective} and the parameterized problem \eqref{eq:para_objective} typically encountered in machine learning. The proofs of the results are found in the Appendix~\ref{sec:proof}.
First, it is easy to show that the optimization over non-singular measures (over $P$) is the same as that overall probably measures  (over $\mathcal{M}$). Moreover, these infima, as one should expect, are at least as good as the infimum for the original problem. The latter can be viewed as a distribution space optimization problem over point masses. This is the content of the next result.

\begin{proposition}\label{prop:inf_relation}
    Suppose that $(x,w) \mapsto u_w(x)$ is
    continuous for $x\in \Omega$, $\gamma$-H\"{o}lder continuous for $w\in W$,
    and $J[\cdot]$ is a convex, $\gamma$-H\"{o}lder continuous functional on $\mathcal{L}^2(\Omega)$, then the following relations  hold:
\begin{align}
    \inf_{\rho \in P} F[\rho]
    =\inf_{\mu \in \mathcal{M}} F[\mu] \le \inf_{w \in W} L(w).
\end{align}
\end{proposition}


A natural follow up question is when equality holds in Prop.~\ref{prop:inf_relation}. It turns out that if $J$ is linear, then we have equality. Moreover, in this case \eqref{eq:measure_objective} can be regarded as a convex dual of \eqref{eq:para_objective}. Detailed discussion on this can be found in Appendix~\ref{sec:proof}. In the more general case, we show in the following that the difference between the optimal loss values of the two formulations depend on how dense $\Ucal_W$ is in its convex hull.

\begin{definition}[$\varepsilon$-dense]
    A subset $\mathcal{V}$ of a metric space $\mathcal{U}$ is $\varepsilon$-dense in $\mathcal{U}$ for a given positive number $\varepsilon$ if for any $u$ in $\mathcal{U}$, there exists $v$ in $\mathcal{V}$ such that the distance between $u$ and $v$ is less than $\varepsilon$.

\end{definition}

\begin{proposition}\label{prop:eps_dense}
    Assume the same conditions in Prop. \ref{prop:inf_relation} and further that $\mathcal{U}_W$ is $\varepsilon$-dense in $\overline\conv(\mathcal{U}_W)$. Then, there exists a $w^* \in W$ and a constant $C$ independent of $w^*$ such that
    \begin{align}
        \inf_{\mu\in \mathcal{M}} F[\mu]
        \le
        \inf_{w\in W} L(w)
        \le
        L(w^*)
        \le
        \inf_{\mu\in \mathcal{M}} F[\mu] + C\varepsilon^\gamma.
    \end{align}
\end{proposition}
As a consequence of Prop. \ref{prop:eps_dense}, if $\Ucal_W$ is an universal approximating class of $\Ucal$, then we would expect $\inf_\mu F(\mu) = \inf_w L(w)$. For example, this is the case for a variety of neural network architectures~\citep{Barron1993Universal,Hornik1989Multilayer,Cybenko1989Approximation,Leshno1993Multilayer}. Note however that this does not imply that every minimizer of $F(\mu)$ must be in the form of a point mass.

\begin{remark}
    Instead of convex combinations, another way to relax~\eqref{eq:para_objective} is to take general linear combinations of functions in $\Ucal_W$. However, in this case, one cannot easily compute the integral over $W$ as Monte Carlo integration is not directly applicable.
\end{remark}

\section{Numerical algorithm}
\label{sec:algo}

In Sec.~\ref{sec:opt_distribution}, we showed that the distribution space problem $\inf_{\rho} F[\rho]$ is a convex formulation for the parameterized problem~\eqref{eq:para_objective}. However, it is still infinite dimensional and some parameterization is required in order to give rise to realizable algorithms. In this section, we present such an approach based on decomposing $\rho$ as a convex combination of ``simple'' distributions
\begin{align}\label{eq:mixture_form}
    \rho_{n,\alpha}(w) := \sum_{i=1}^{n} \alpha_i \phi_i(w),
    \qquad
    \alpha \in \Lambda_n,
\end{align}
where $\Lambda_n$ denotes the $n$-dimensional probability simplex, i.e. $\alpha \in \Lambda_n$ if $\alpha_i \in [0,1]$ for all $i=1,\dots,n$ and $\sum_{i=1}^n \alpha_i = 1$
Here, the $\phi_i$'s are chosen, fixed distributions which are simple to sample from, and act as building blocks for approximating more complex distributions. Therefore, $\inf_\rho F[\rho]$ can be replaced by $\min_\alpha F[\rho_{n,\alpha}]$, i.e. an optimization problem involving coefficients of the mixture distribution. Note that this is still a convex optimization problem in $\alpha$, but is now finite dimensional. Fig.~\ref{fig:function_space} illustrates the overall approach: the original problem is over a non-convex set $\Ucal_W$ (Fig.~\ref{fig:function_space}(a)), which we extend to $\overline\conv(\Ucal_W)$ (Fig.~\ref{fig:function_space}(b)). Finally, we restrict it to a $n$-dimensional problem in $\Ucal_n := \{ \int_{W} u_w \rho_{n,\alpha}(w) dw: \alpha \in \Lambda_n  \}$. This approach is \emph{consistent} if $\phi_i$'s are picked in such a way that $\min_\alpha F[\rho_{n,\alpha}] \rightarrow \inf_\rho F[\rho]$ as $n\rightarrow\infty$. There are of course many ways to ensure this, and in the following we outline one such method based on scaling and translating simple distributions.

\begin{figure}[thb!]
    \centering
    \subfigure[Direct parameterization]{\includegraphics[width=0.4\textwidth]{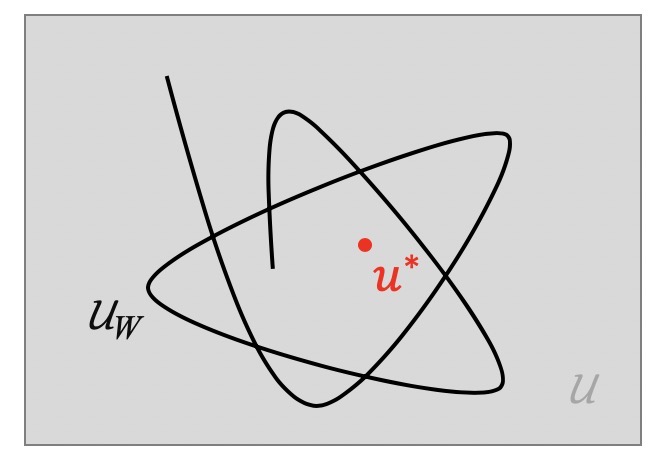}}
    \subfigure[Convex combination]{\includegraphics[width=0.4\textwidth]{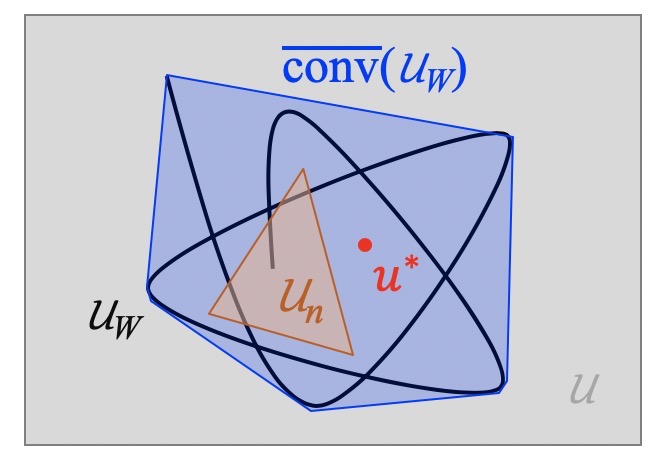}}
    \caption{
        Distribution space approach. We start with the non-convex parameter space problem (a), which we then convexify to a distribution-space problem in $\overline\conv(\Ucal_W)$ (b), and finally we restrict it to a finite dimensional problem via mixture distribution approximation on $\Ucal_n$.
    }
    \label{fig:function_space}
\end{figure}

We begin with a definition of the \emph{approximate identity}, which is a class of scaled versions of an envelope function that can be used to approximate a point mass.
\begin{definition}
    Let $\varphi \in C_c(W)$
    be a probability density function which is non-negative, uniformly continuous, 
    satisfies
    $\int_{\mathbb{R}^d} \varphi(w) dw = 1$, and
    $\varphi(w) \le \beta(1+\|w\|)^{-d-\eta}$ for some positive constants $\beta$ and $\eta$.
    For each $0<\sigma\le1$, define $\varphi_\sigma(w):=\sigma^{-d} \varphi(\sigma^{-1}w)$. The family of functions $\{\varphi_\sigma(w)\}$ is called an approximate identity. 
\end{definition}
In one dimensional case, scaled triangular hat functions and truncated Gaussian functions are examples of approximate identities.
Now, we form our parameterization of $P$ by the following mixture family $P_\varphi$
\begin{align}
    P_\varphi = \Big\{\rho(w):  \rho(w) = \sum_{i=1}^n \alpha_i \varphi_{\sigma_i}(w-\mu_i), \mu_i \in \mathbb{R}^d, \sigma_i \in (0,1], \alpha \in \Lambda_n,  n \in\mathbb{N}^+\Big\}.
\end{align}
Note that in the definition of $P_\varphi$, one can further restrict $\mu_i \in \mathbb{Q}^d$ and $\sigma_i \in \{1/k: k\in\mathbb{N}^+\}$ using similar techniques in mixture distribution approximations~\citep{Nestoridis2011Universal,Bacharoglou2010Approximation}. We now show that the class $P_\varphi$ is sufficiently large, and a $n$-term approximation from it is consistent as $n\rightarrow \infty$.

\begin{theorem}[Approximation]\label{th:approximation_infF}
    Under the same conditions as in Prop.~\ref{prop:inf_relation}, we have
\begin{align}
    \inf_{\rho \in P_\varphi} F[\rho]
    = \inf_{\rho \in P} F[\rho]
    =\inf_{\mu \in \mathcal{M}} F[\mu].
\end{align}
Furthermore, for any fixed accuracy $\varepsilon>0$, there exists a constant $C>0$ and a $n$-component mixture density $\rho_{n,\alpha} \in P_\varphi$ such that
\begin{align}\label{eq:approx_eps_n}
    F[\rho_{n,\alpha}] - \inf_{\rho \in P} F[\rho]
    \le \big(\varepsilon + \tfrac{C}{n}\big)^{\gamma/2}, \forall n \in \mathbb{N}^+.
\end{align}
\end{theorem}

The preceding analysis shows that the mixture distribution approximation not only preserves convexity but is also consistent (Thm.~\ref{th:approximation_infF}).
In other words, given a large enough $n$, a $n$-component mixture distribution can be used to represent $\rho$ approximately, thus we can instead solve a finite-dimensional optimization problem to determine the optimal mixture coefficients. This is the basic idea of our proposed algorithm, and in the following, we shall describe more concretely its various computational aspects.

The first point is that of inference. In usual machine learning models, given a sample $x$, inference amounts to simply evaluating the trained function $x \mapsto u_{w^*}(x)$ to produce a prediction. In our formulation, however, we obtain from training a distribution $\rho^*$ over $W$, and the inference step is
\begin{align}
    x \mapsto \int_W u_w(x) \rho^*(w) dw
    \approx
    \frac{1}{R} \sum_{i=1}^{R} u_{w_i} (x),
    \qquad
    w_i \overset{i.i.d.}{\sim} \rho^*,
\end{align}
where $R$ is taken sufficiently large so that the Monte-Carlo integration is accurate enough. In particular, this requires repeated sampling of the distribution $\rho^*$.
In the approach based on the mixture of distributions outlined previously, sampling can be done by first drawing the mixture index $i$ according to the mixture coefficients $\alpha \in \Lambda_n$ and then drawing from the basic distribution $\phi_i$ itself.

On the other hand, the training algorithm is rather simple.
Here we assume that a good mixture basis $\{\phi_i(w):i=1,..,n\} \subset P_\varphi$ is chosen. Hence, we are solving
\begin{align}\label{eq:min_prob_simplex}
    \min_{\alpha \in \Lambda_n} & ~ l(\alpha):=F[\rho_{n,\alpha}]
    \equiv
    F
    \left[
        \sum_{i=1}^{n} \alpha_i \phi_i
    \right].
\end{align}
Observe that this is a finite-dimensional constrained convex optimization problem and can be solved efficiently using projected gradient descent, for which efficient algorithms for projections onto the probability simplex can be used~\citep{Duchi2008Efficient, Wang2013Projection}. The additional complication is that we must use Monte-Carlo integration to estimate the gradients. The algorithm is summarized in Alg.~\ref{alg:A}.

    \begin{remark}
        Assuming the integral $\psi_i(x) := \int u_w(x) \phi_i(w) dw$ can be computed exactly, the optimization problem (\ref{eq:min_prob_simplex}) reduces to the classic approximation using basis $\{\psi_i(x)\}$. However, $\psi_i$ is implicitly given and influenced by either network architecture $u_w$ or the distribution $\phi_i$. Instead of direct methods, Monte Carlo integration can be employed to compute integrals involving $\phi_i(w)$.
    \end{remark}

\begin{algorithm}
    \caption{Projected gradient descent}
    \label{alg:A}
    \begin{algorithmic}
    \STATE Set $\{\phi_i(w):i=1,..,n\}$,
    and the numerical parameters
    $R \in \mathbb{N}^+, k_{max}\in \mathbb{N}^+$, $\epsilon >0$, $\varepsilon>0$.
    \STATE Initialize $k=0$, and $\alpha^{(0)} \in \Lambda_n$.
    \REPEAT
    \STATE Set $\bar u=0$ and randomly choose $R$ samples, $i_1, ...,i_R$, from $\{1,..,n\}$ with probability $\alpha^{(k)}$.

    \FOR{$r = 1$ to $R$}
        \STATE Get a sample $w_r$ form $\phi_{i_r}$, and sum $\bar u = \bar u + \tfrac 1R u_{w_r}$
    \ENDFOR

    \STATE Calculate the gradients:
    $g^{(k)} = \big\langle \tfrac{\delta J}{\delta u}(\bar u),
    \langle u_w, \phi_i\rangle \big\rangle$.

    \STATE Normalize the gradients: $\tilde g^{(k)} = \tfrac{1}{n}(g^{(k)} - \text{mean}(g^{(k)})) / \text{std}(g^{(k)})$.

    \STATE Update and project $\alpha$ to the probability simplex $\Lambda_n$: $\alpha^{(k+1)} = \text{Proj}_{\Lambda_n} \big(\alpha^{(k)} - \varepsilon \tilde g^{(k)} \big)$.
    \STATE $k = k + 1$.
    \UNTIL{$k=k_{max}$ or $\|\tilde g_k\| < \epsilon$}
    \end{algorithmic}
\end{algorithm}

\section{Numerical experiments}
\label{sec:numerics}

In this section, we demonstrate the practical implementation of Alg.~\ref{alg:A}. The overarching goal of these experiments, designed to be simple in nature, is to highlight the key properties of the distribution-space algorithm as expected from our preceding analyses:
\begin{enumerate}
    \item \textbf{Approximation.} As more basic distributions are used in the mixture approximation ($n$ increases), the results become better. In particular, it may exceed direct optimization over the parameter space if the latter is difficult (e.g. highly non-convex), consistent with Prop.~\ref{prop:inf_relation} and Thm.~\ref{th:approximation_infF}.
    \item \textbf{Robustness.} The optimization problem over $\alpha$ should be more stable and insensitive to initialization since it is a convex problem, and we can control its complexity by limiting $n$.
\end{enumerate}
Of course, the distribution space formulation has its own computational drawbacks compared to the traditional approach. First, computation overheads are high. This is because even at inference time, multiple sampling and function evaluation iterations are required (See Appendix~\ref{sec:appendix_experiments} for some quantitative experiments demonstrating this). Furthermore, the computation of the (stochastic) gradient $g^{(k)}$ in Alg.~\ref{alg:A} also requires a number of samplings, which can slow down the training. Finally, the choice of the mixture basis $\{ \phi_i \}$ is not obvious and has to be designed for each application. Thus, a general and comprehensive study of these issues in large-scale practical problems is out of the scope of the current paper and will be the topic of subsequent work. The focus of the following experiments is to demonstrate 1, 2 outlined above, and to highlight the differences between the distribution-space algorithm from the traditional optimization approach in parameter space. Moreover, we wish to gain some insights, where possible, into the origins of these differences in simple settings.

\begin{figure}[thb!]
    \subfigure[1D regression, smooth]{
        \begin{minipage}[r]{0.42\textwidth} \centering 
            \quad\includegraphics[width=0.9\textwidth]{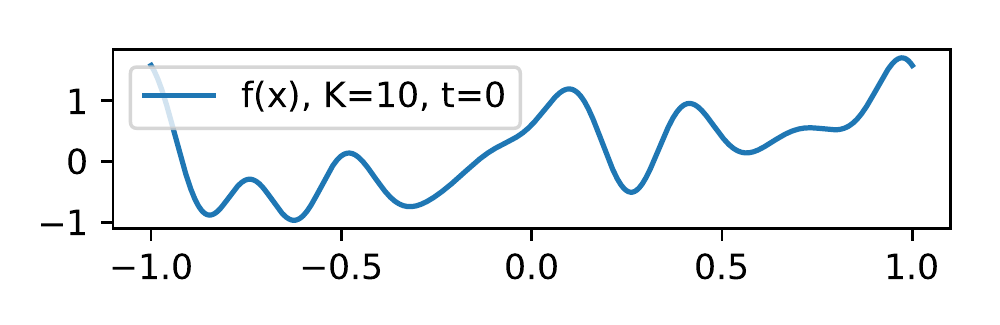}\\
            \includegraphics[width=\textwidth]{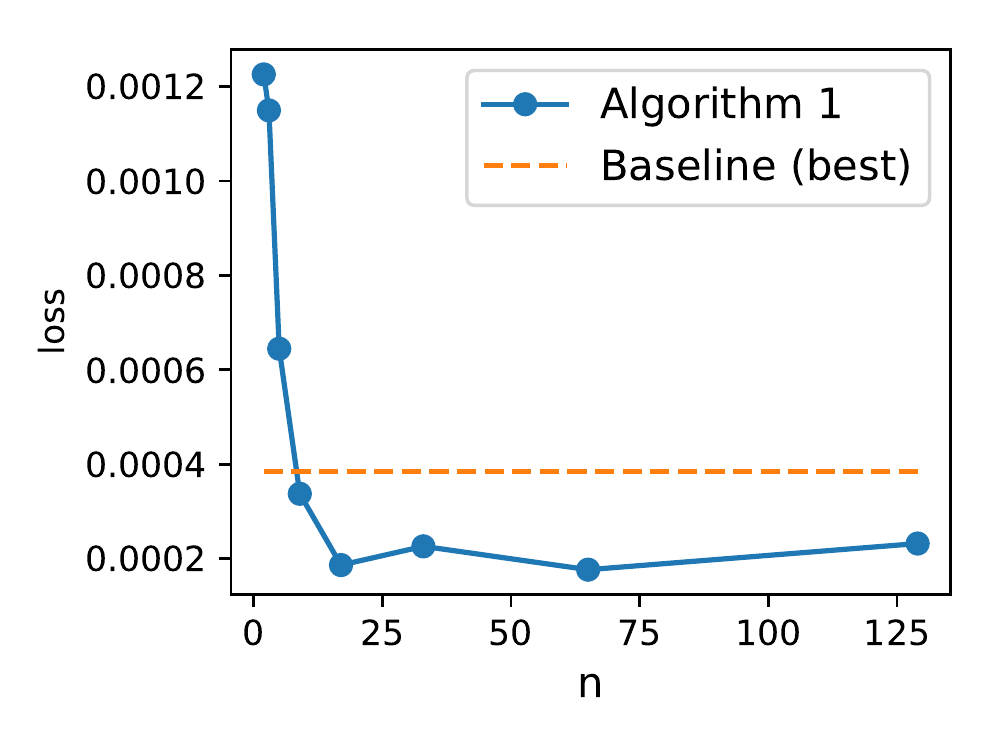}
        \end{minipage}
    }%
    \quad\quad\quad
    \subfigure[1D regression, jump]{
        \begin{minipage}[r]{0.42\textwidth} \centering
            \quad\includegraphics[width=0.9\textwidth]{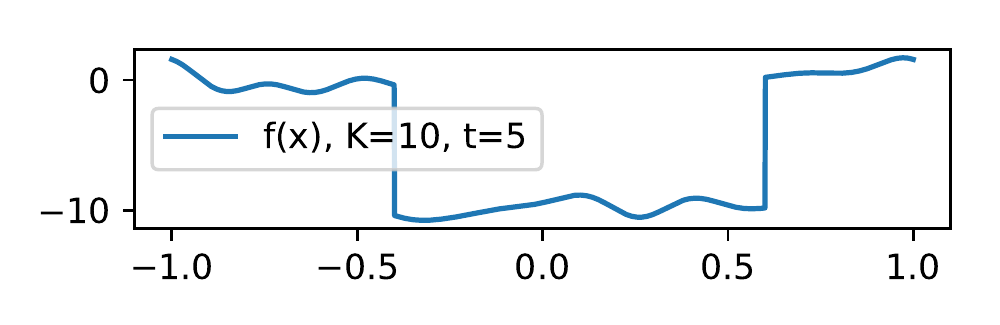}\\
            \includegraphics[width=\textwidth]{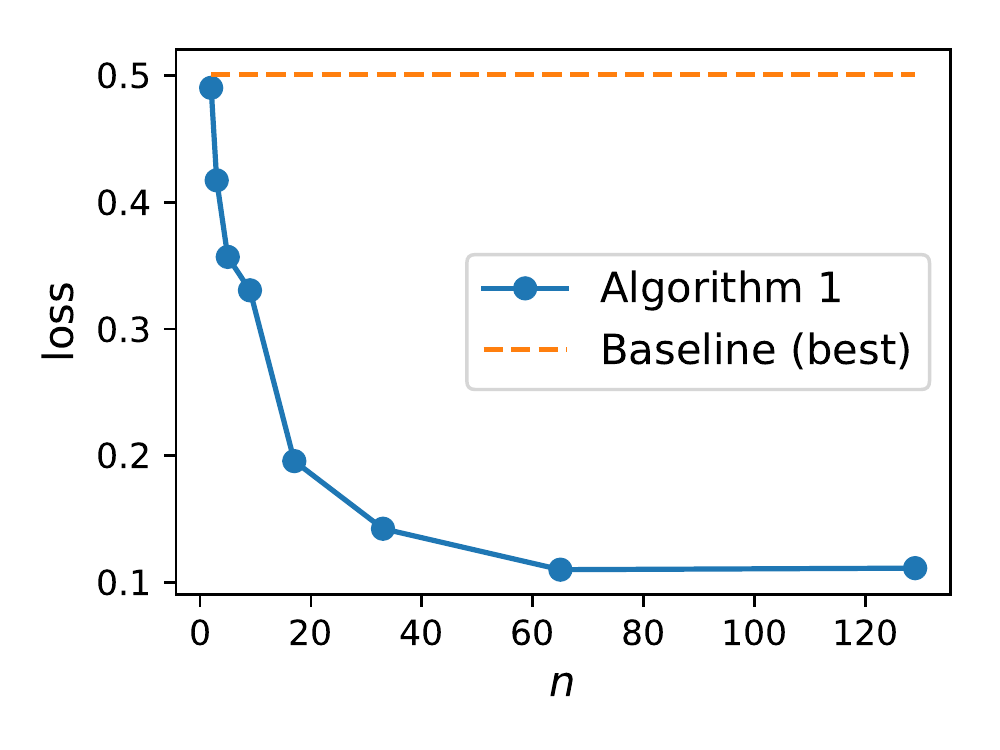}
    \end{minipage}
    }%

        \subfigure[MNIST, $N=100$]{
        \quad\includegraphics[width=0.46\textwidth]{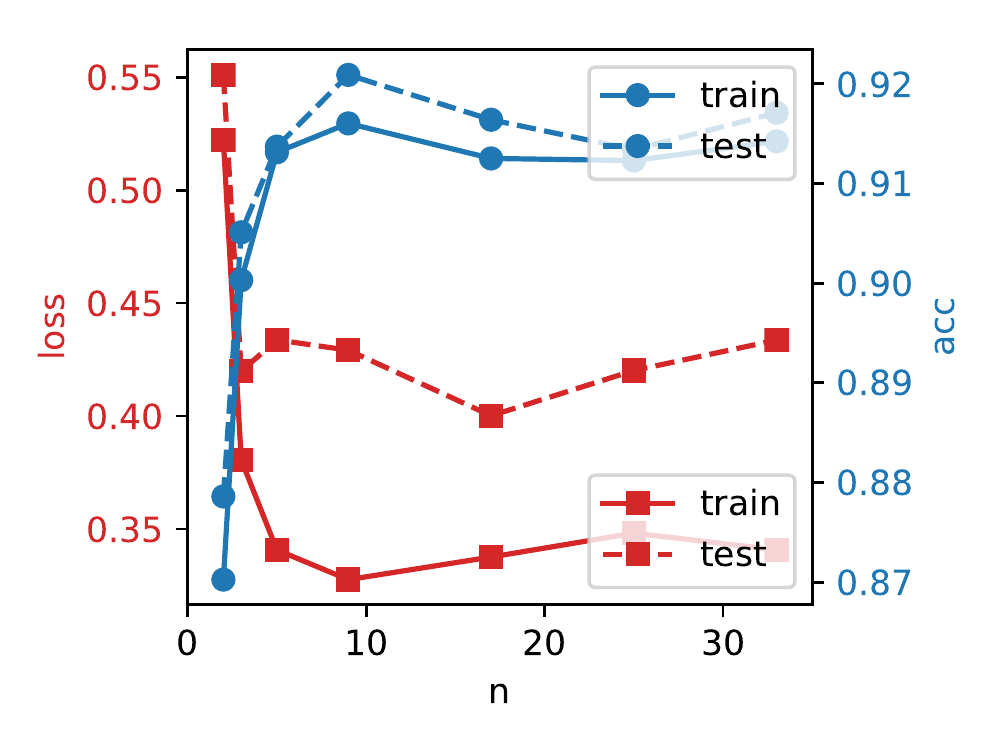}
        }
    \subfigure[MNIST, CNN]{
        \includegraphics[width=0.46\textwidth]{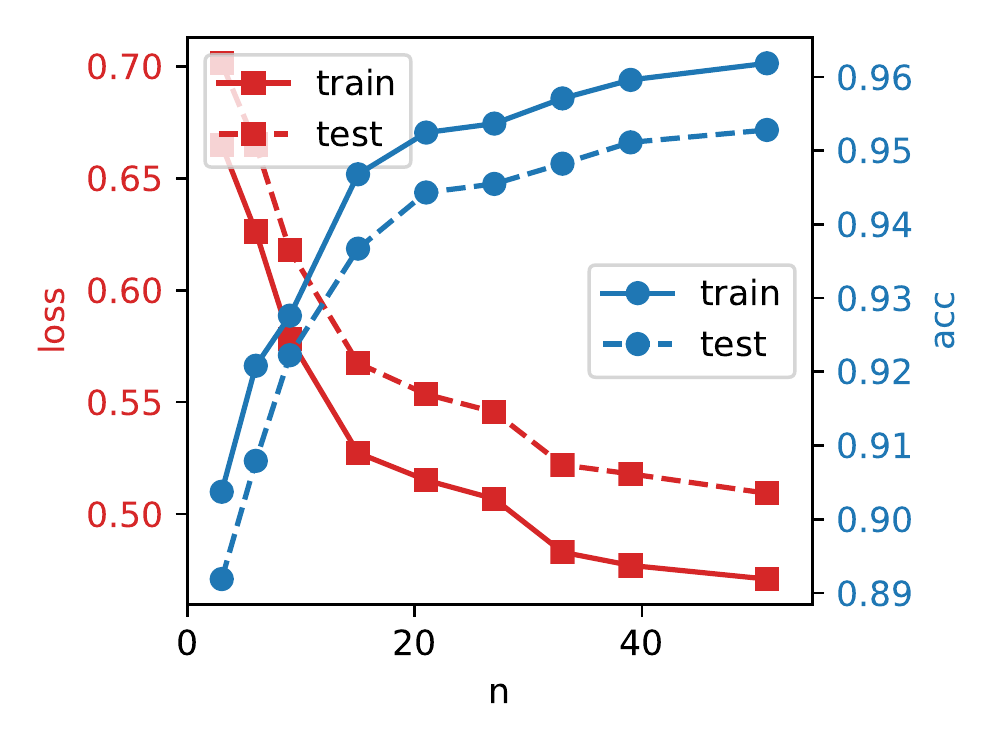}
        }
    \caption{Results demonstrating approximation properties. For all cases we see that approximation improves for increasing $n$. For 1D regression (a,b), for $n$ large enough Alg.~\ref{alg:A} outperforms the baseline, which is direct optimization in parameter space using the SGD optimizer where the initialization and learning rates are well tunned. The improvement is more significant for the case with jumps in $f$ (b). The same approximation properties is observed in MNIST classification (c) and (d) (blue for {\color{blue}accuracy}, red for {\color{red}loss}). ($R=20, k_{max}=3000,\varepsilon=0.1$ for 1D regression and $R=10,k_{max}=300,\varepsilon=0.1$ for MNIST classification are fixed.)}
    \label{fig:approximation}
\end{figure}

\begin{figure}[thb!]
    \center
    \subfigure[1D regression]{\includegraphics[width=0.4\textwidth]{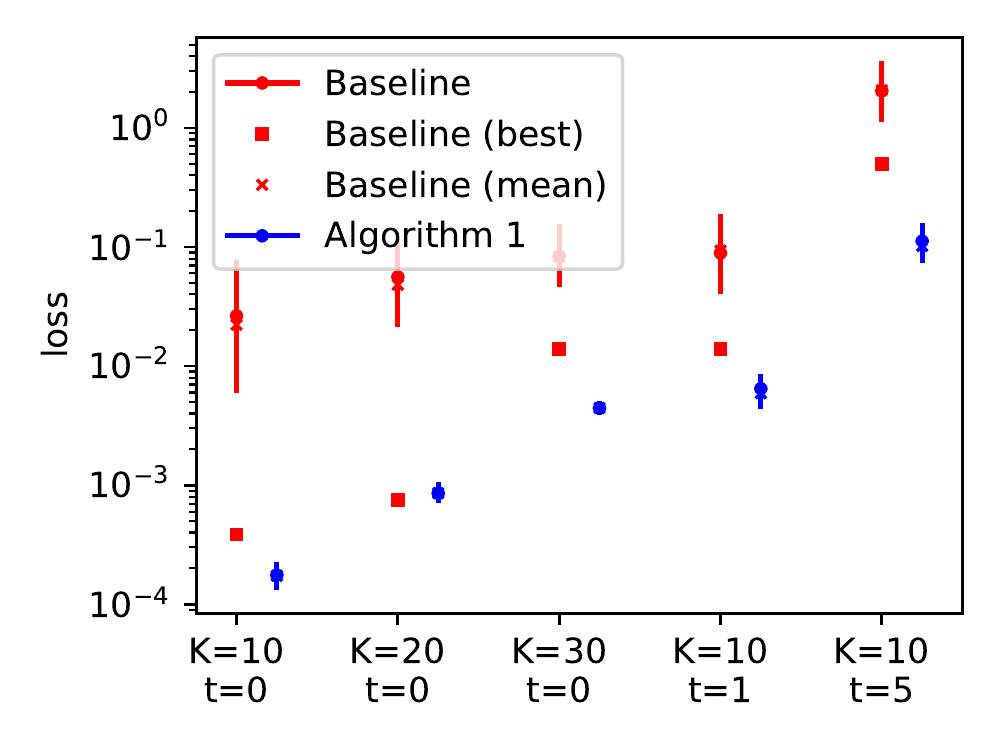} 
    }\\
    \subfigure[MNIST, loss]{\includegraphics[width=0.4\textwidth]{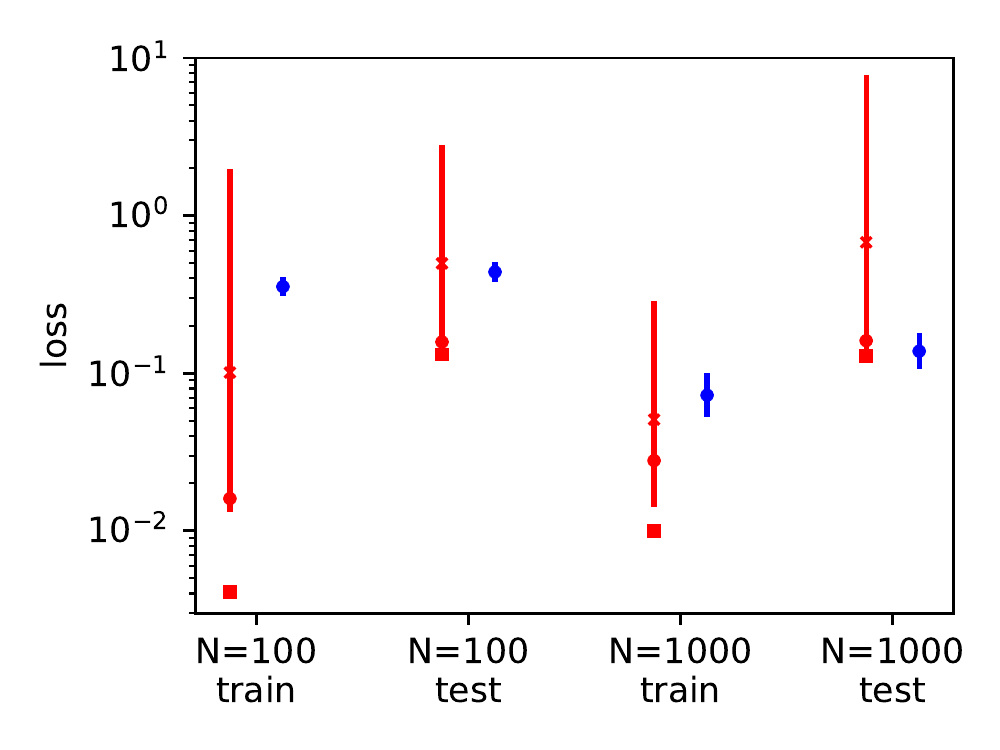}
        }
    \subfigure[MNIST, accuracy]{\includegraphics[width=0.4\textwidth]{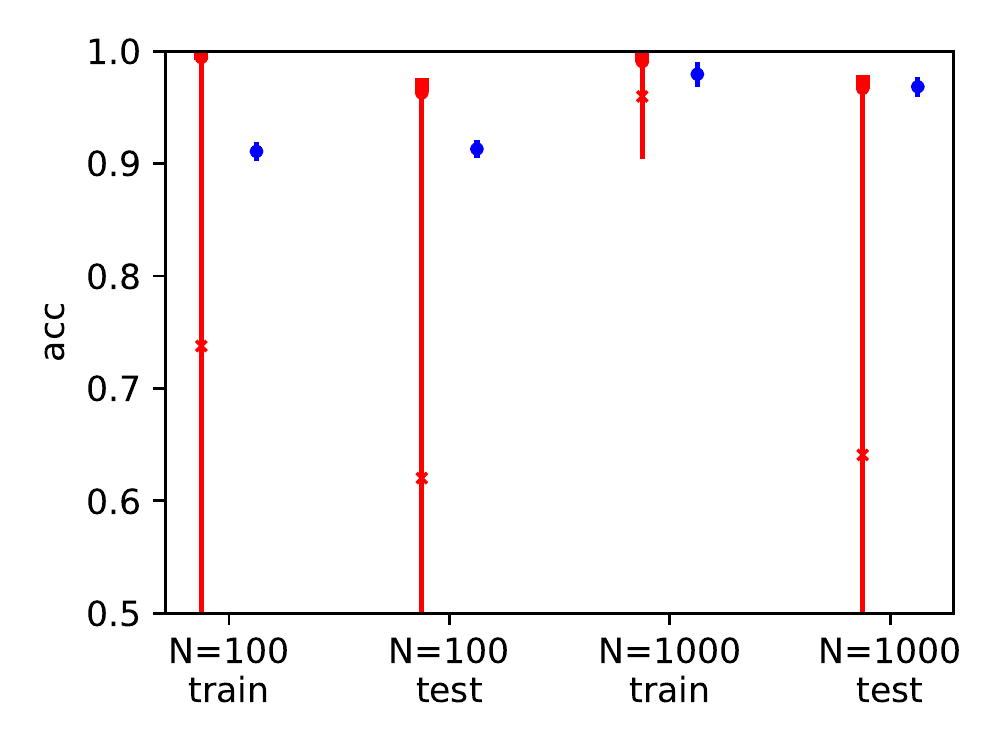}
        }\\
    \subfigure[MNIST, CNN, loss]{\includegraphics[width=0.4\textwidth]{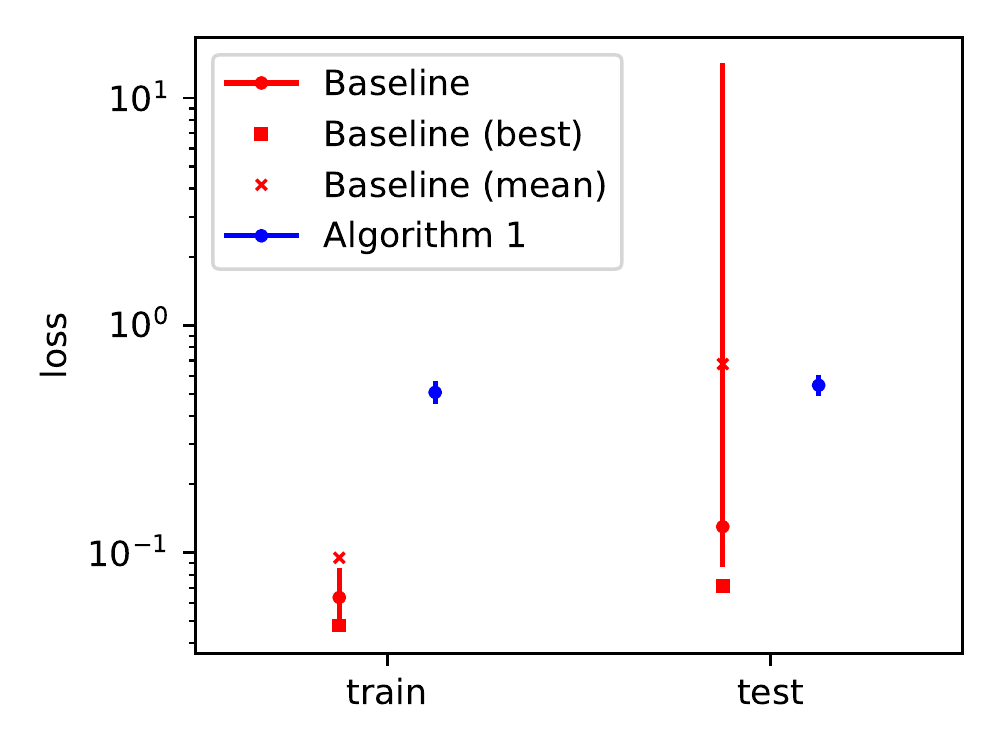}
        }
    \subfigure[MNIST, CNN, accuracy]{\includegraphics[width=0.4\textwidth]{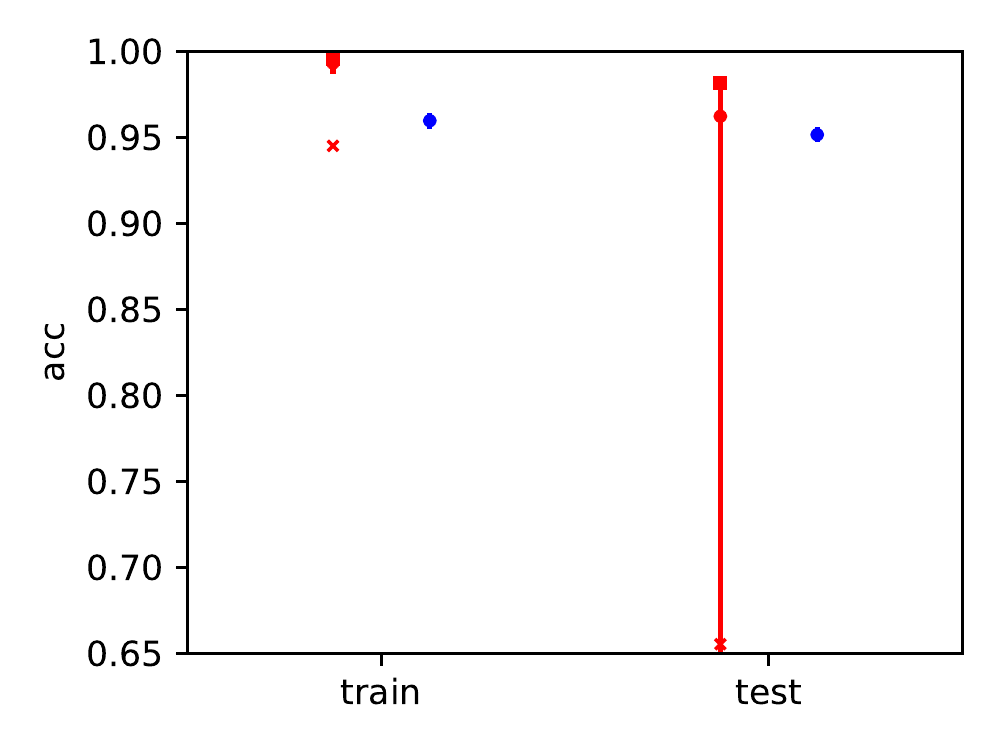}
        }
    \caption{
        Results demonstrating robustness properties.
        In each run, the initial conditions are sampled randomly according to some typical distributions (see Appendix~\ref{sec:appendix_experiments} for details). The interquartile range is shown as error bars. (a) 1D regression. We observe that Alg.~\ref{alg:A} outperforms the SGD baseline, especially for $f$ with large jumps. Furthermore, the distribution-space results are essentially independent of initial condition (error bars show the fluctuation of Monte Carlo integration with $R=20$).
        The same holds for the MNIST experiments (b,c,d,e), except that for small networks $N=100$ (b), the mean (test) performance is worse than direct training. Nevertheless, the robustness of Alg.~\ref{alg:A} is evident.
    }
    \label{fig:robustness}
\end{figure}

\subsection{1D regression problem}

First, we consider a toy example in the form of 1D least-squares regression. Given some target function $f: [-1,1] \rightarrow\R$, we approximate it using a one-hidden-layer neural network with ReLU activations, i.e. $\sum_{j=1}^N a_j \relu( v_j x + b_j)$. Due to positive homogeneity of $\relu$, we may reduce the number parameters by setting $v_j = \cos(\theta_j), b_j = \sin(\theta_j), \theta_j\in [0,2\pi]$, yielding the parameter-space non-convex optimization problem
\begin{align}\label{eq:1d_min_w}
    \begin{split}
        &\min_{w\in W}
        J[u_w]
        \equiv
        \|
        u_w - f
        \|^2_{L^2([-1,1])},
        \\
        &u_w(x) = \sum_{j=1}^N a_j \relu( \cos(\theta_j) x + \sin(\theta_j)), \\
        &w \equiv (a, \theta) \in W \equiv [-A,A]^N \times [0,2\pi]^{N},
    \end{split}
\end{align}
where $A$ is a sufficiently large positive number.

For the distribution-space counterpart, we further simplify~\eqref{eq:1d_min_w} by writing it as $\min_\theta J[u_{\theta}]$ where $u_{\theta} := \argmin_a J[u_{(a,\theta)}]$, eliminating the need for sampling $a$. Hence, we obtain
\begin{align}
    \inf_{\rho \in P}
    F[\rho]
    \equiv
    \left\|
        \int_{\Theta} u_{\theta} \rho(\theta) d\theta
        - f
    \right\|^2_{L^2([-1,1])},
    \qquad
    \Theta := [0,2\pi]^N.
\end{align}

Now, we have to choose a collection of basic distributions $\{\phi_i, i=1,\dots,n\}$. Here, we perform a dimensional reduction technique typically used in density function theory (for example see \citep{Eschrig1996fundamentals}) where we write
\begin{align}\label{eq:nested_inf}
    \inf_\rho F[\rho] = \inf_{\bar{\rho}} \inf_{\rho: \rho\text{'s marginal is } \bar{\rho}} F[\rho],
\end{align}
with $\bar{\rho}$ denoting a distribution over $[0,2\pi]$. Note that equality holds since the neural network is invariant to permutations of its nodes, thus optimality is obtained at a distribution whose marginals are all equal. In addition, the optimization problem in the right hand of \eqref{eq:nested_inf} is still convex (see Appendix~\ref{sec:proof}).
We now assume that the inner minimizer of the right hand side is well-approximated by a product distribution $\bar{\rho}(\theta_1)\times \dots \times \bar{\rho}(\theta_N)$, each of which we then approximate by the mixture form~\eqref{eq:mixture_form} with
\begin{align}
    \phi_i(\theta) =
    \tfrac{n}{2\pi}\varphi
    \left(
        \tfrac{n}{2\pi}\theta - i+1
    \right),
    \qquad
    i=1,...,n,
\end{align}
where $\varphi$ is the triangle function $(1-|x|)\chi_{[-1,1]}(x)$. Optimization proceeds as in Alg.~\ref{alg:A}. Notice that strictly speaking, after the dimensional reduction the problem in $\alpha$ need not remain convex, but now the distribution space is only over 1D distributions, which greatly simplifies computations.

We considered the regression target function $f$ represented as a Fourier series with random coefficients and truncated at $K$ terms. Moreover, we introduce two jumps located at $x=-0.4,0.6$, whose scales are controlled by a parameter $t \in \mathbb{R}$,
\begin{align}
    f(x) = \sum_{k=0}^{K-1} \frac{1}{1+k}\big(
        c_k \cos(2k\pi x) + s_k \sin(2k\pi x) \big)
        + (\text{sign}(x-0.6) - \text{sign}(x+0.4)) t,
\end{align}
where the coefficients $c_k,s_k$ are independently sampled from the standard normal distribution. The graph of $f$ are found in the top insets of Fig.~\ref{fig:approximation}(a,b). In the following experiments, $K,t$ will be varied, noting that large $t$ increases non-smoothness and large $K$ introduces higher frequencies.
The training data contains 3000 samples uniformly distributed $[-1,1]$, and their labels set to the value of $f$ at these points. Note that for each sample $\theta \in \Theta$, $a\in \R^N$ is computed by solving the least-squares problem $\min_a J[u_{(a,\theta)}]$.

Fig.~\ref{fig:approximation}(a-b) validates the approximation property: as the number of mixture component $n$ increases, the error becomes smaller, and the performance of the distribution-space algorithm exceeds that of direct optimization, especially so for target functions which have large jumps. In this case, direct optimization is difficult as a greater resolution in $\Theta$ space is required, corresponding to points where the change in $f$ or its derivatives are significant, e.g., at jump discontinuities. See Appendix~\ref{sec:appendix_experiments} for some discussion on this point.

Fig.~\ref{fig:robustness}(a) demonstrates the robustness of the distribution space algorithm compared to the baseline algorithm (SGD in parameter space) for a variety of initializations and learning rates. As expected, the baseline parameter-space results are sensitive to initializations, whereas Alg.~\ref{alg:A} is practically independent of initial conditions. Furthermore, as expected, its improvement versus the best baseline results is more pronounced when the jumps in $f$ are large (see Fig.~\ref{fig:approximation}(b)).

\subsection{MNIST classification with one-hidden layer neural networks}

We now test the distribution-space algorithm on the slightly larger MNIST classification problem~\citep{lecun1998mnist} with softmax cross-entropy loss. We consider a one-hidden-layer neural network again, but now with cosine activations. This is inspired by the random feature model~\citep{Rahimi2008Random}, but in our case, instead of using a fixed distribution, we train the distribution governing the feature map parameters. Concretely, we consider the model
\begin{align}
    u_w(x) = \sum_{j=1}^{N} a_j \cos(v_j^T x + b_j) + c,
    \qquad
    a_j,c \in \R^{10}, v_j \in \R^{784}, b_j\in\R,
\end{align}
and as before, we alleviate the need to sample $(a,c)$ by assuming that $(a,c)$, given $(v,b)$ and the data, is chosen to minimize the loss. For approximating the distribution over $v$ and $b$, we choose the basic distribution family
\begin{align}
    \phi_i(v, b)
    \propto
    \exp{
        (
            -\tfrac{\lambda_i^2}{2}
            \| v \|^2
        )
        \chi_{[0,2\pi]^N}(b)
    }
\end{align}
where $\lambda_i$'s are picked between 1 and 20. That is, the basic distributions are standard normals on $v$ and uniform on $b$. This choice is motivated by the random feature approach~\citep{Rahimi2008Random}. However, instead of choosing $\lambda_i$ as a hyper-parameter, here we optimize over mixture distributions with different $\lambda_i$'s. As before, given a sample $(v,b)$, the vector $a$ and $c$ are determined by minimizing $J[u_{(a,c,v,b)}]$. This is no longer a least-squares problem, hence we solve it approximately by applying 2 epochs of Adam \citep{Kingma2014Adam} on $(a,c)$. Note that unlike the previous example, here the current problem is always convex in the mixture coefficients since we did not perform the dimensional reduction technique.

Fig.~\ref{fig:approximation}(c) again validates the fact that as $n$ increases, the approximation improves. However, as our collection $\{\phi_i\}$ is rather constrained, in the small network considered here ($N=100$), the performance is worse than the parameter-space baseline, which has around 97\% accuracy. Fig.~\ref{fig:robustness}(b-c) demonstrates the robustness of Alg.~\ref{alg:A} -- we can see that this advantage is more pronounced than the 1D example, which is to be expected since the higher dimensional problem typically has a more complex optimization landscape. However, we note that for this specific application and network architecture, it is not hard to come up with a good initialization scheme for the parameter-space optimization. Nevertheless, Fig.~\ref{fig:robustness}(b-c) demonstrates that arbitrary initialization can lead to very poor performance, which is not the case for the distribution-space algorithm. Finally, notice in Fig.~\ref{fig:robustness}(b-c) that for larger networks ($N=1000$), the small mixture basis here appears to be enough to rival the performance of optimization in parameter-space. This suggests that for very large networks, a simple representation in the distribution space over trainable parameters may be obtained, which is somewhat consistent with recent observations in~\citep{Martin2018Implicit,martin2019traditional}.

\subsection{MNIST classification with multi-layer neural networks}

As further validation, we consider the MNIST classification problems using a multi-layer convolutional neural network, which contains two convolutional layers with leaky-ReLU activation followed by a fully connected layer with softmax activation (the full specification is displayed in Appendix~\ref{sec:mnist_cnn}). The weights $w$ of this network contain two parts, $w_{\text{CNN}}$ and $w_{FC}$, which are the weights in the convolutional layers and the fully connected layer, respectively. Unlike networks with only one hidden layer, multi-layer networks need more sophisticated means to choose the basis distributions $\{\phi_i(w)\}$.
The main reason is that the weights in different layers are dependent on one another, and a na\"{i}ve parameterization treating them as equivalent is not sufficient.
Instead, we construct $\{\phi_i(w)\}$ as follows,
\begin{align}
    \phi_i(w) \equiv \phi_i(w_{\text{CNN}}, w_{FC})
    \propto
    \tilde \phi_i(w_{\text{CNN}}) P(w_{FC} | w_{\text{CNN}};\sigma,\text{data}),
\end{align}
where $\tilde \phi_i(w_{\text{CNN}})$ is the density function for $w_{\text{CNN}}$, and
$P(w_{FC} | w_{\text{CNN}})$ is the conditional density for $w_{FC}$ with given $w_{\text{CNN}}$. Particularly, for any fixed $w_{\text{CNN}}$, the $w_{FC}$ is initialized (element-wise) by i.i.d. Gaussian $\mathcal{N}(0,\sigma^2)$ with fixed $\sigma$, and then updated by one epoch Adam-optimizer with learning rate $10^{-3}$ on 5000 MNIST samples (about 10\% of the MNIST dataset). All that remains now is constructing the density $\tilde \phi_i(w_{\text{CNN}})$, taking into account the dependence across different layers.

A popular way to parameterize high dimensional distributions is using the \emph{reparameterization trick}: instead of directly parameterizing $\tilde{\phi}_i$ itself, we model samples from $\tilde{\phi}_i$ as i.i.d. random Gaussian noise transformed by a nonlinear function parameterized by a neural network. This approach has been used in VAEs~\citep{Kingma2013Auto} and GANs~\citep{Goodfellow2014Generative} to approximate complex distributions. In the specific case of CNNs, there is a line of works on generating weights for a CNN using another network, known as hyper-networks \citep{Ha2016Hypernetworks, Krueger2017Bayesian, Deutsch2018Generating, Ratzlaff2019HyperGAN}, which automatically takes care of dependence across layers using the reparameterization trick. Here, we adopt the hyper-network architecture in \citep{Deutsch2018Generating} to model $\tilde \phi_i(w_{\text{CNN}})$. For appropriate comparison, we do not train the hyper-networks on the MNIST dataset itself. Instead, we aim to model some generic feature extractor distributions on deep CNN weights, and we accomplish this by borrowing ideas from transfer learning: we train the hyper-network on the fashion-MNIST dataset~\citep{Xiao2017Fashion} with a variety of hyper-parameters, which produces a collection of $\{ \tilde{\phi}_i \}$ that are then used as basis distributions for the MNIST task.
In particular, the hyper-network is trained by minimizing a loss function $\tilde L(\Phi)$ which is a compromise between accuracy and diversity (quantified by the negative of the entropy of the outputs),
where a hyper-parameter $\lambda$ is used to balance the accuracy loss and diversity loss. Choosing different values of $\lambda$ will give different distributions on $w_{\text{CNN}}$. In addition, the hyper-network at different training steps are corresponding to different distributions, hence we can enumerate the distributions $\{\tilde \phi_i(w_{\text{CNN}})\}$ as a collection of hyper-networks with different $\lambda$ and different training steps $k_{\text{hyper}}$. Particularly, we set $\lambda \in [10^2,10^3]$ and $k_{\text{hyper}} \in \{5000,10000,20000\}$. The full specification can be found in Appendix~\ref{sec:mnist_cnn}.
%
Fig.~\ref{fig:approximation}(d) again validates the fact that as $n$ increases, the approximation improves, and Fig.~\ref{fig:robustness}(d-e) demonstrates the robustness of Alg.~\ref{alg:A}.

Let us emphasize an important point these experiments demonstrate about the distribution-space approach. First, it is clear that in high dimensions, it is not feasible to cover all interesting target functions with a fixed, generic mixture family, due to the curse of dimensionality. 
Consequently, the performance of the approach depends on a good choice of $\{ \phi_i \}$, and this should be adapted to the problem. In the preceding examples, we showed that with an appropriate choice of $\{ \phi_i \}$, $n$ need not be very large in order to attain good performance. This leads to an interesting direction worthy of future exploration, namely the relationship between model architectures ($\Ucal_w$) and the ease of optimization in distribution space. The previous findings appear to suggest that large, over-parameterized models may possess this property.

\section{Conclusion}
\label{sec:conclusion}

In this paper, we discuss an alternative formulation of non-convex optimization problems in machine learning as convex problems in the space of probability distributions over the parameters. Through mixture distribution approximation, we developed algorithms akin to ``subspace methods'' to perform optimization directly in distribution space. We showed using simple examples that due to their convex or low dimensional nature, such distribution-space methods have advantageous properties such as insensitivity to initialization. Moreover, well-known dimensional reduction techniques in distribution space can be applied. Overall, this paves the way for an alternative theoretical and algorithmic approach towards large-scale optimization in machine learning.

%% file: tex/appendix.tex
\section{Proof of results}
\label{sec:proof}

In this section, we give the proofs of the claims in the main text.

\subsection{Continuous density function approximation}

\begin{proposition}[Revised Prop. \ref{prop:inf_relation}]
    Suppose that $(x,w) \mapsto u_w(x)$ is
    continuous for $x\in \Omega$, $\gamma$-H\"{o}lder continuous for $w\in W$,
    and $J[\cdot]$ is a convex, $\gamma$-H\"{o}lder continuous functional on $\mathcal{L}^2(\Omega)$, then the following relations  hold:
\begin{align}
    \inf_{\rho \in P} F[\rho]
    =\inf_{\mu \in \mathcal{M}} F[\mu] \le \inf_{w \in W} L(w).
\end{align}
\end{proposition}
\begin{proof}
    The inequality part is a direct consequence of Jensen inequality. For the equality part, it is obvious that $\inf_{\rho \in P} F[\rho] \ge \inf_{\mu \in \mathcal{M}} F[\mu]$, hence we only need to prove the following claim: for any $\mu\in\mathcal{M}$ and any $\varepsilon>0$, there exist $\rho\in P$ such that $F[\rho] \le F[\mu] + \varepsilon$.

    We will prove the claim. Note that, for any $\mu \in \mathcal{M}$, the function $h(x) := \int_W u(x,w) d\mu(w)$ is Lebesgue measurable and $h(x) \in \mathcal{L}^2(\Omega)$.

    Since $J[\cdot]$ is $\gamma$-H\"{o}lder continuous, for any $\varepsilon$, there is a a constant $C_\gamma>0$, such that
\begin{align}\label{eq:JJ_ineq}
    | J[g]-J[h] | \le C_\gamma \|g-h\|^\gamma \le \varepsilon,
\end{align}
for any $g(x)\in \mathcal{L}^2(\Omega)$. Therefore, we only need to prove the existence of $g(x) := \int_W u(x,w) \rho(w) dw$ with $\rho(w)\in P$ such that $\|g(x)-h(x)\|\le \delta := (\varepsilon/C_\gamma)^{1/\gamma}$.

    In fact, consider $\varphi$ as the Gaussian function, then $\rho_\sigma(w) := \int \varphi_\sigma(w-v)d\mu(v) \in P$ is a continuous density function for any $\sigma>0$.
    Let $g_\sigma(x) := \int_W u(x,w) \rho_\sigma(w) dw $, then we have
\begin{align}
    |g_\sigma(x) - h(x)| &=
    \Big|\int_W u(x,w) \rho_\sigma(w) dw - \int_W u(x,w) d\mu(w) \Big|\\
    &= \Big| \int_W \int_W (u(x,w)-u(x;v)) \varphi_\sigma(v-w) dv d\mu(w) \Big| \\
    &\le \Big| \int_W \int_W C_{w,\gamma} \|w-v\|^{\gamma} \varphi_\sigma(v-w) dv d\mu(w) \Big| \\
    &\le
    C_{w,\gamma} C_{\varphi,\gamma} \sigma^\gamma =: C \sigma^\gamma,
\end{align}
where $C_{w,\gamma}$ is the constant in the $\gamma$-H\"{o}lder condition of $u_w$, $C_{\varphi,\gamma}$ is the $\gamma$-moment of $\varphi$.
Let $\sigma$ small enough such that $C \sigma^\gamma |\Omega| \le \delta$, then we have $\|g_\sigma(x)-h(x)\| \le \delta$ which finishes the proof.
\end{proof}

To prove the Prop. \ref{prop:eps_dense}, we only need to prove the following lemma.
\begin{lemma}
    Suppose $J[\cdot]$ is a convex $\gamma$-H\"{o}lder continuous functional and $\mathcal{U}_W$ is $\varepsilon$-dense in $\overline\conv(\mathcal{U}_W)$. Then, there exists a $w^* \in W$ and a constant $C$ independent of $w^*$ such that
    \begin{align}
        \inf_{\mu\in \mathcal{M}} F[\mu]
        \le
        L(w^*)
        \le
        \inf_{\mu\in \mathcal{M}} F[\mu] + C\varepsilon^\gamma.
    \end{align}
\end{lemma}
\begin{proof}
    The first inequality is obvious. For the second inequality, let $\mu^* \in \mathcal{M}$ such that $F[\mu^*] \le \inf_{\mu\in \mathcal{M}} F[\mu] + \varepsilon^\gamma$, define
     $g(x) := \int_W u(x,w) d\mu^*(w)$,
     then there is a positive number $\delta>0$ and a constant $C_\gamma$ such that
    $| J[g]-J[u] | \le C_\gamma \|g-u\|^\gamma$.

    Since $\mathcal{U}_W$ is $\varepsilon$-dense in $\overline\conv(\mathcal{U}_W)$ and $g \in \overline\conv(\mathcal{U}_W)$, we can find $w^*$ such that
    $\|g(x)-u(x,w^*)\| \le \varepsilon$. Hence, choosing $C=C_\gamma +1$, we have
    \begin{align}
        L(w^*) \equiv J[u(x,w^*)] \le F[\mu^*] + C_\gamma \varepsilon^\gamma
        \le \inf_{\mu\in \mathcal{M}} F[\mu] + C \varepsilon^\gamma,
    \end{align}
    which finishes the proof.
\end{proof}

\subsection{Linear case}

If $J[\cdot]$ is a linear functional, then $F[\mu]$ can be rewritten as
$F[\mu]:=\int_W L(w) d\mu(w)$. The optimization problem of $F[\mu]$ over $\mu\in\mathcal{M}$ can be regarded as a dual problem of the optimization of $L(w)$ over $w\in W$. In fact, minimization of $L(w)$ is equivalent to the following problem,
\begin{align}\label{eq:reformulate_minL}
    \max_{\eta\in\mathbb{R}} \eta, ~s.t. ~L(w) \le \eta,
\end{align}
whose Lagrangian dual problem is the optimization of $F[\mu]$. Furthermore, the linearity of $J[\cdot]$ implies the following equalities,
\begin{align}
    \inf_{\rho \in P} F[\rho]
    =\inf_{\mu \in \mathcal{M}} F[\mu] = \inf_{w \in W} L(w).
\end{align}
Particularly, when $L(w)$ has global minimizers in $W$, we have the following result.
\begin{proposition}
    \label{th:min_over_measure}
    If $L(w)$ is a continuous function which has at least a global minimizer in $W$(with minimum $l_0$), i.e. the set  $M = \arg\min\limits_{w\in \Omega} L(w) $ is not empty. Then for any probability measure $\mu$ in $\mathcal{M}$, we have $F[\mu]:=\int_\Omega L(w) d\mu(w) \ge l_0$, where the equality is satisfied if and only if $\operatorname {supp}(\mu) \subset M$.
\end{proposition}
\begin{proof}

The inequality is apparent since $\mu$ is a probability measure and $ L(w) \ge l_0$. Hence we only need to show the condition to get equality.

(1) On one hand, if $\operatorname {supp}( \mu ) \subset M$, then $F[\mu] = \int_M L(w) d\mu(w) = \int_M l_0 d\mu(w) = l_0$.

(2) One the other hand, if  $A:=\operatorname {supp}(\mu) \setminus M$ is not empty, where both $M$ and $\operatorname {supp}(\mu)$ are closed set, then
\begin{align}
    F[\mu] = \int_{A} L(w) d\mu(w) + \int_{\operatorname {supp}(\mu) \cap M} L(w) d\mu(w)
= l_0 + \int_{A} (L(w)-l_0) d\mu(w) > l_0.
\end{align}
Here we used the definition of $A$ and the fact that $L(w) > l_0$ for all $w \in A $.
\end{proof}

\subsection{Density approximation}

It is well known that any continuous distribution can be approximated arbitrarily well by a finite mixture of normal densities \citep{Sorenson1971Recursive}.
This approximation property can be extended to the approximate identities \citep{Zeevi1997Density, Nestoridis2011Universal, Bacharoglou2010Approximation}. Particularly, an $O(1/n)$ approximation order can be obtained by a result attributed to Maurey and proved, for example, in Barron \citep{Barron1993Universal}. One application of those results is the following lemma, which is useful for our analysis.

\begin{lemma}\label{lemma:P_phi_approx}
    For any probability density function $\rho(w) \in P$ and $\varepsilon>0$, there exists $\rho_\varepsilon(w) \in P_\varphi$ such that $\|\rho_\varepsilon(w) - \rho(w)\| < \varepsilon$. Furthermore, there exists a constant $C>0$ and $n$-component mixture densities $\rho_{n}(w) \in P_\varphi$ such that
    $\|\rho_{n}(w) - \rho(w)\|^2 \le \varepsilon + \tfrac{C}{n}, \forall n \in \mathbb{N}^+$.
\end{lemma}
\begin{proof}
    We prove the claims by four steps below.

    (1) When $\sigma$ is small enough, the following density $\rho_\sigma$ satisfies $\|\rho_\sigma(w)-\rho(w)\|^2 \le \tfrac\varepsilon8$,
    \begin{align}\label{eq:rho_sigma}
        \rho_\sigma(w) := \int_W \rho(v) \varphi_\sigma(w-v) dv.
    \end{align}
    In fact, since $\rho$ is uniformly continuous, for any $\varepsilon_0>0$, there is a constant $\delta_0>0$ such that $|\rho(w)-\rho(w-v)|\le \varepsilon_0$ if $\|v\|\le\delta_0$. Hence we have the following inequality,
    \begin{align*}
        |\rho(w) - \rho_\sigma(w)| &= \Big|\int (\rho(w)-\rho(w-v)) \varphi_\sigma(v) dv \Big| \\
        &\le
        \int_{\|v\|<\delta_0} |\rho(w)-\rho(w-v)| \varphi_\sigma(v) dv
        +
        \int_{\|v\|\ge\delta_0} |\rho(w)-\rho(w-v)| \varphi_\sigma(v) dv\\
        &\le
        \varepsilon_0+ 2\max_w \rho(w) \int_{\|v\|\ge\delta_0} \varphi_\sigma(v) dv.
    \end{align*}
    According to the definition of $\varphi, \varphi_\sigma$, $\int_{\|v\|\ge\delta_0} \varphi_\sigma(v) dv$ can be arbitrary small if $\sigma$ is small enough. Therefore, we can choose a $\sigma$ such that $|\rho(w)-\rho_\sigma(w)| \le 2\varepsilon_0, \forall w\in W$. Particularly, let $4\varepsilon_0^2|W| = \tfrac{\varepsilon}{8}$, then we have $\|\rho_\sigma(w)-\rho(w)\|^2 \le \tfrac\varepsilon8$.

    (2) The Riemann summation for the integral in \eqref{eq:rho_sigma} can approximate $\rho_\sigma$ uniformly. In fact, for any $m$-partition $\{B_i: i=1,...,m\}$ of $W$, one can define 
    \begin{align}\label{eq:rho_sigma_m}
        \rho_{\sigma, m}(w) = \sum_{i=1}^m \rho(v_i) |B_i| \varphi_\sigma(w-v_i), \quad v_i \in B_i, s.t.~
        \rho(v_i)|B_i| = \int_{B_i} \rho(v) dv.
    \end{align}
    Then,
    \begin{align*}
        |\rho_{\sigma, m}(w)-\rho_\sigma(w)|
        &=
        \Big|
        \sum_{i=1}^m \big(\int_{B_i} \rho(v)\varphi_\sigma(w-v)dv 
        -\rho(v_i) |B_i| \varphi_\sigma(w-v_i) \big) \Big|\\
        &\le 
        \sum_{i=1}^m \Big|\int_{B_i} \big(\rho(v)\varphi_\sigma(w-v) 
        -\rho(v_i) \varphi_\sigma(w-v_i)\big) dv \Big|\\
        &\le 
        \sum_{i=1}^m \Big|\int_{B_i} (\rho(v)-\rho(v_i)) \varphi_\sigma(w-v) dv \Big|\\
        &\quad+
        \sum_{i=1}^m \Big|\int_{B_i} \rho(v_i)(\varphi_\sigma(w-v)-\varphi_\sigma(w-v_i)\big) dv \Big|.
    \end{align*}
    Since $\rho, \varphi_\sigma$ are uniformly continuous functions, for any $\delta>0$, there exist a partition $\{B_i\}$ small enough (and $m$ large enough) such that
    \begin{align*}
        |\rho_{\sigma, m}(w)-\rho_\sigma(w)|\le 
        \tfrac\delta2 \sum_{i=1}^m \Big|\int_{B_i} \varphi_\sigma(w-v) dv \Big|+\tfrac\delta 2
        \sum_{i=1}^m \Big|\int_{B_i} \rho(v_i) dv \Big|
        \le \delta, \forall w \in W,
    \end{align*}
    which implies $\|\rho_{\sigma,m}-\rho\|^2 \le \delta^2 |W|$.

    (3)  According to \eqref{eq:rho_sigma_m}, $\rho_{\sigma,m}$ is a finite convex combination of $\{\varphi_{\sigma}(w-v_i):i=1,...,m\}$ where $\|\varphi_\sigma(w-v_i)\|^2$ is bounded by a constant $C_\sigma$ only dependent on $\sigma$. Using Barron's result~\citep{Barron1993Universal}, there exist an $n$-term convex combination $\rho_{n}$ such that $\|\rho_n-\rho_{\sigma,m}\|^2 \le \tfrac{C_\sigma^2}{n}$.

    (4) Finally, let $8\delta^2|W|=\varepsilon, C := 2 C^2_\sigma$, then we have
    \begin{align}
        \|\rho-\rho_n\|^2 \le 4\|\rho-\rho_\sigma\|^2
        +4\|\rho_\sigma-\rho_{\sigma,m}\|^2
        +2\|\rho_{\sigma,m}-\rho_n\|^2 \le \varepsilon + \tfrac{C}{n}.
    \end{align}

\end{proof}

\begin{lemma}
    Under the same conditions as in Prop.~\ref{prop:inf_relation}, we have
\begin{align}\label{eq:equality_in_lemma}
    \inf_{\rho \in P_\varphi} F[\rho]
    = \inf_{\rho \in P} F[\rho].
\end{align}
Furthermore, for any $\rho(w) \in P$ and $\varepsilon>0$, there exists a constant $C>0$ and a $n$-component mixture density $\rho_{n} \in P_\varphi$ such that
\begin{align}
    F[\rho_{n}] - F[\rho]
    \le \big(\varepsilon + \tfrac{C}{n}\big)^{\gamma/2}, \forall n \in \mathbb{N}^+.
\end{align}

\end{lemma}
\begin{proof}
    We only need to prove the second part which implies the equality \eqref{eq:equality_in_lemma}. Note that $u(x,w)$ is bounded for all $x\in \Omega$ and $w\in W$. Hence for any $\rho, \tilde\rho \in P$, there exists a constant $C_\rho$ such that
    \begin{align}
        |F[\tilde\rho] - F[\rho]|
        \le C_\gamma \Big\| \int_W u(x,w) (\tilde\rho(w)-\rho(w)) dw\Big\|^\gamma
        \le
        C_\rho^\gamma \|\tilde\rho-\rho\|^\gamma.
    \end{align}
    According to the second part of Lemma \ref{lemma:P_phi_approx}, there exists a constant $C>0$ and $n$-component mixture densities $\rho_{n}(w) \in P_\varphi$ such that
    $\|\rho_{n}(w) - \rho(w)\|^2_2 \le \tfrac{\varepsilon}{C^2_\rho} + \tfrac{C/C^2_\rho}{n}, \forall n \in \mathbb{N}^+$, hence we have:
    \begin{align}
        |F[\rho_{n}] - F[\rho]|
        \le
        C_\rho^\gamma \|\rho_n-\rho\|^\gamma
        \le \big(\varepsilon + \tfrac{C}{n}\big)^{\gamma/2}, \forall n \in \mathbb{N}^+.
    \end{align}
\end{proof}

Combining the results in Prop. \ref{prop:inf_relation}, we have the following corollary.
\begin{corollary}
    Under the same conditions as in Prop.~\ref{prop:inf_relation}, for any $\varepsilon>0$, there exists a constant $C>0$ and a $n$-component mixture density $\rho_{n} \in P_\varphi$ such that
\begin{align}\label{eq:approx_eps_n}
    F[\rho_{n}] - \inf_{\mu \in \mathcal{M}} F[\mu]
    \le \big(\varepsilon + \tfrac{C}{n}\big)^{\gamma/2}, \forall n \in \mathbb{N}^+.
\end{align}
\end{corollary}

\subsection{Nested minimization}

For a convex functional $F_N[\rho_N]$ on $N$-dimension density function $\rho_N$, we define $F[\rho]$ as a functional on 1-dimension density function $\rho$,
\begin{align}
    F[\rho] = \min_{\rho_N : \rho_N \text{'s marginal is}~ \rho}  F_N[\rho_N].
\end{align}
Further, if we parameterize $\rho$ by a convex combination of some density functions $\phi_i$,
\begin{align}
    \rho = \sum_{i=1}^n \alpha_i \phi_i,
\end{align}
then the loss $F[\rho]$ becomes a function $l$ on $\alpha$. One merit of this definition is the convexity of $F_N[\rho_N]$ can be inherited.
\begin{proposition}
    If $F_N[\rho_N]$ is convex with respect to $\rho_N$, then $F[\rho]$ is convex with respect to $\rho$, and $l(\alpha)$ is convex with respect to $\alpha$.
\end{proposition}
\begin{proof}
    It is enough to show that for any numbers $a_1,a_2\in (0,1), a_1+a_2=1$, and probability densities $\rho_1, \rho_2$, the following inequality holds,
    \begin{align}
        F[a_1 \rho_1 + a_2 \rho_2]
        \le a_1 F[\rho_1] + a_2 F[\rho_2].
    \end{align}
    In fact, we have
    \begin{align*}
        F[a_1 \rho_1 + a_2 \rho_2] &=
        \inf_{\rho_N \sim a_1 \rho_1 + a_2 \rho_2} F_N[\rho_N] \\
        &\le
        \inf_{\rho_{N,i} \sim \rho_i} F_N[a_1 \rho_{N,1} + a_2 \rho_{N,2}]\\
        &\le
        a_1 \inf_{\rho_{N,1} \sim \rho_1} F_N[\rho_{N,1}]
        +
        a_2 \inf_{\rho_{N,2} \sim \rho_2} F_N[\rho_{N,2}]\\
        &=
        a_1 F[\rho_1] + a_2 F[\rho_2],
    \end{align*}
    which finishes the proof.
\end{proof}

\newpage

\section{Details of experiments}
\label{sec:appendix_experiments}

\subsection{1D regression}

\textbf{Direct optimization. }
In the tests of direct optimization in parameter space, we randomly initialize the parameters $(\theta,a)$ with $a_i$ sampled from normal distribution $\mathcal{N}(0,\sigma_a^2)$ and  $\theta_i=\arctan(\xi_i) + \eta_i \pi$ where $\xi_i$ sampled from uniform distribution on $[-1,1]$ and $\eta_i \in \{0,1\}$ obeys Bernoulli$(\tfrac12)$ distribution. We train the neural network with 100 hidden nodes using SGD optimizer with batch size 200, max-step 300000 and learning rate $\varepsilon_\theta$ and $\varepsilon_a$ for $\theta$ and $a$ respectively. Hence the training process contains three hyper-parameters, $\sigma_a, \varepsilon_a$ and $\varepsilon_\theta$, which need to be tunned. In addition, we use leaky-ReLU instead of ReLU to improve training by reducing dead neurons, and the following configuration of hyper-parameters are considered:
\begin{align}
    (\sigma_a, \varepsilon_a, \varepsilon_\theta) \in
    \{1,2,4,8,16,32\} \times
    \{\text{logspace(-6,0,11)}\} \times
    \{\text{logspace(-6,-3,7)}\},
\end{align}
where the best configuration for $f(x)$ with $K=10, t=0$ is $(8, 0.063, 3.16\text{e-}5)$ (see Fig.~\ref{fig:1d_N100_profiles}(a)).
In Fig.~\ref{fig:1d_N100_SGD}, the losses of configurations with $\varepsilon_\theta = 3.16\text{e-}5$ are given. Note that the loss is small at large learning rate for $a$, but the performances is not stable, for example the iteration diverges at $\varepsilon_a=1$.
\begin{figure}[thb!]
    \center
    \includegraphics[width=0.7\textwidth]{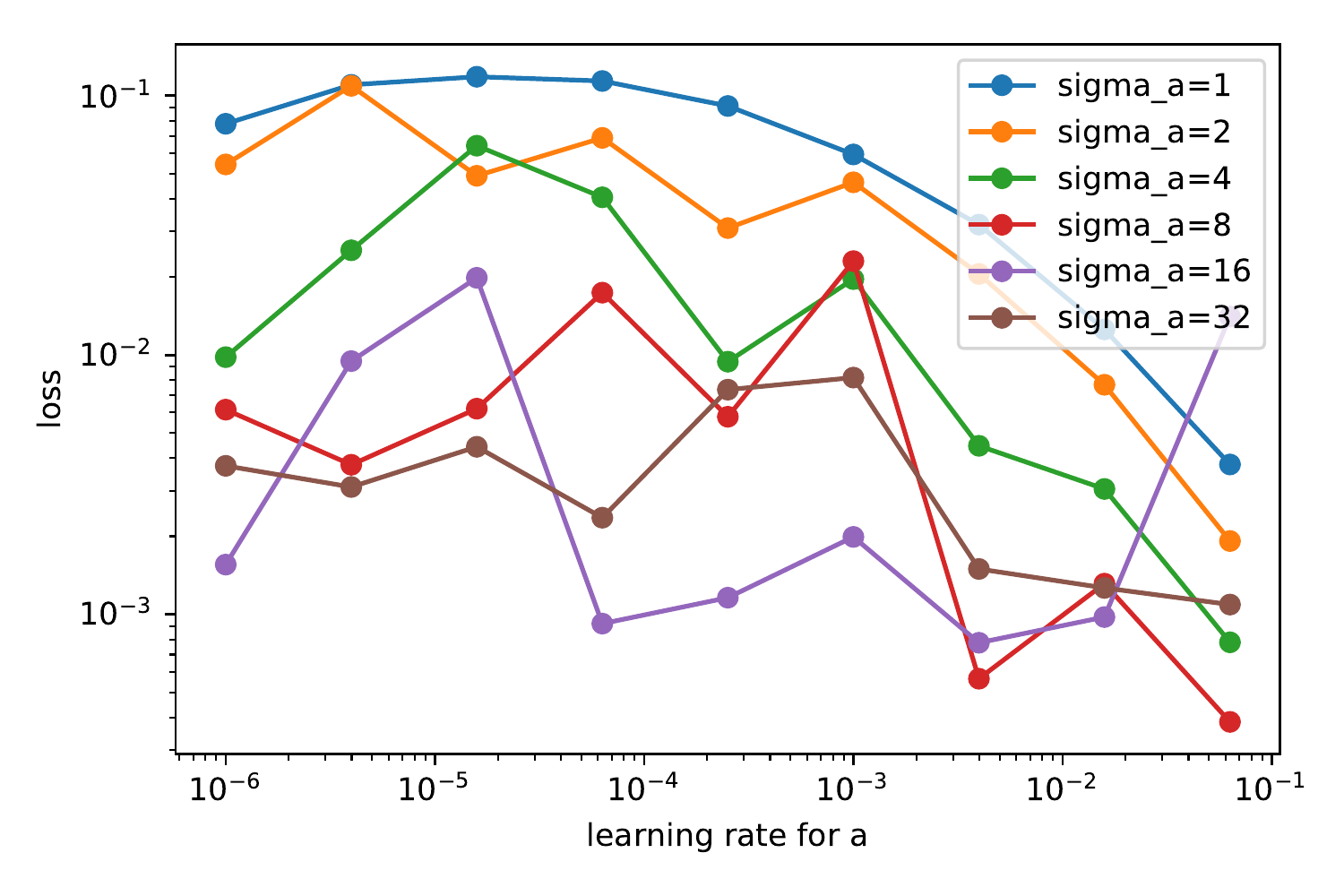}
    \\
    \caption{Direct optimization in parameter space with 100 hidden nodes. The losses of configurations with learning rate $ \varepsilon_\theta = 3.16\text{e-}5$. }
    \label{fig:1d_N100_SGD}
\end{figure}

For the parameter space problem, which is highly non-convex, it is expected that a proper choice of hyper-parameters is crucial to get good performances.
We observed two factors that significantly influence the performances. On the one hand, since our data $x$ are located in $[-1,1]$, it is evident that the nodes with $\tan(\theta_i) \not \in [-1,1]$ have little contribution to learning; therefore we need a small learning rate to prevent nodes leaving the critical region. On the other hand, the scale of $a_i$ effects the output magnitude of the neural network, hence it needs to match the magnitude of $f(x)$, especially when the learning rate is small, in which case the scale of $a_i$ hardly changes overtraining process. As a result, hidden nodes tend to bunch up together (in regions of significant changes in $f$ or its derivatives), which reduce the number of effective nodes to approximate $f$ effectively (see Fig.~\ref{fig:1d_N100_profiles}(b)).

\begin{figure}[thb!]
    \centering
    \subfigure[The best configuration]{
        \includegraphics[width=0.46\textwidth]{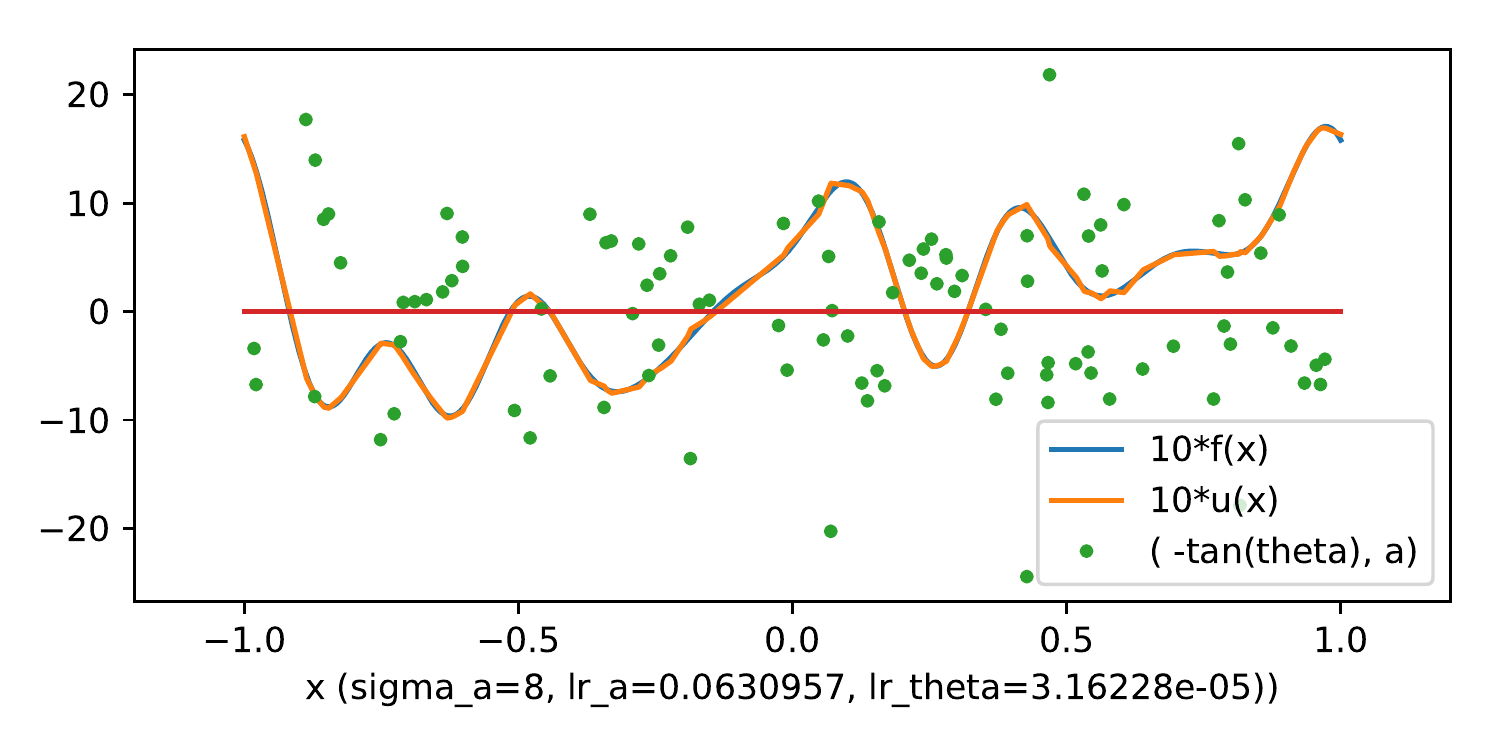}
        }
    \subfigure[An example configuration]{
        \includegraphics[width=0.46\textwidth]{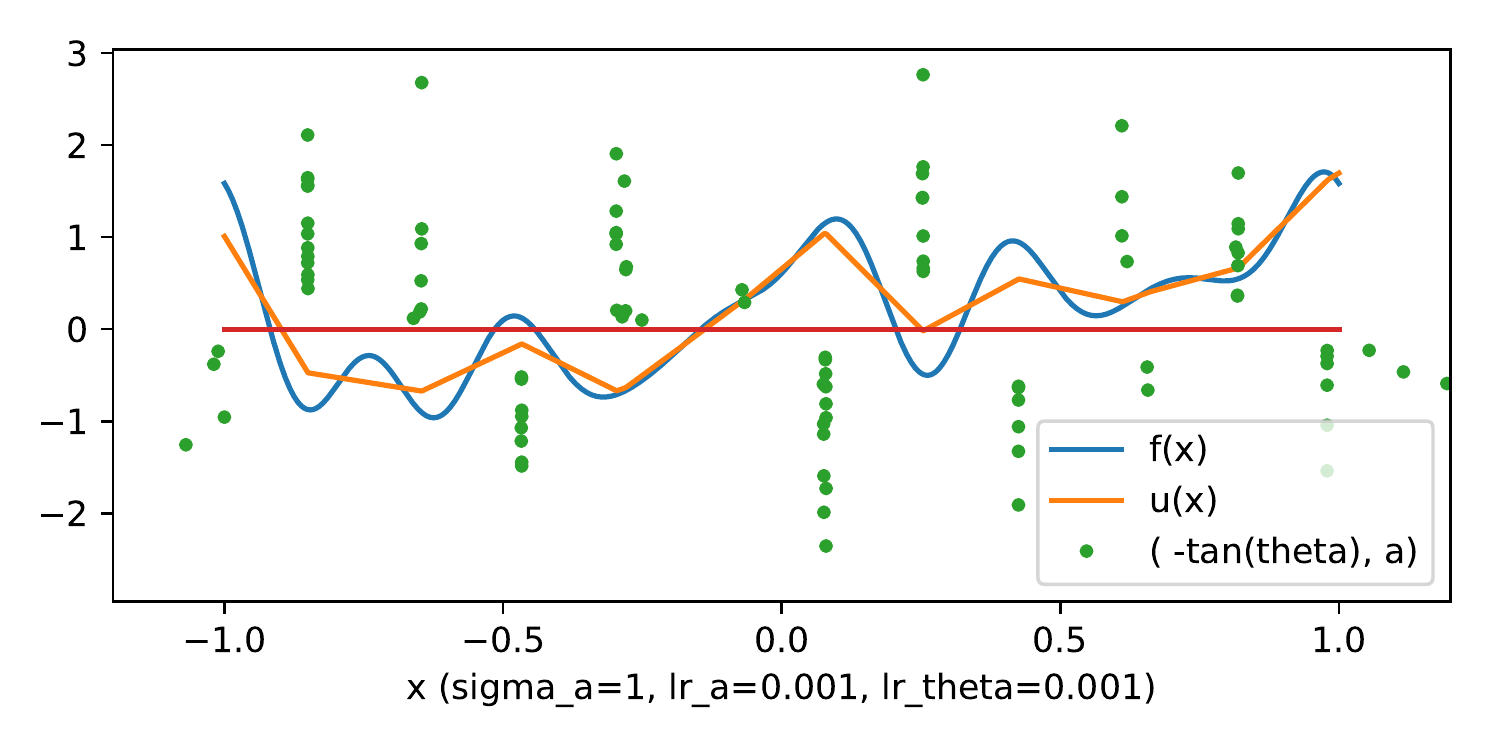}
        }
    \caption{Direct optimization in parameter space with 100 hidden nodes. Profile of the best configuration (a) and a configuration where the nodes go up together (b).}
    \label{fig:1d_N100_profiles}
\end{figure}

\begin{figure}[thb!]
    \centering
    \subfigure[$K=10,t=0$]{
        \includegraphics[width=0.46\textwidth]{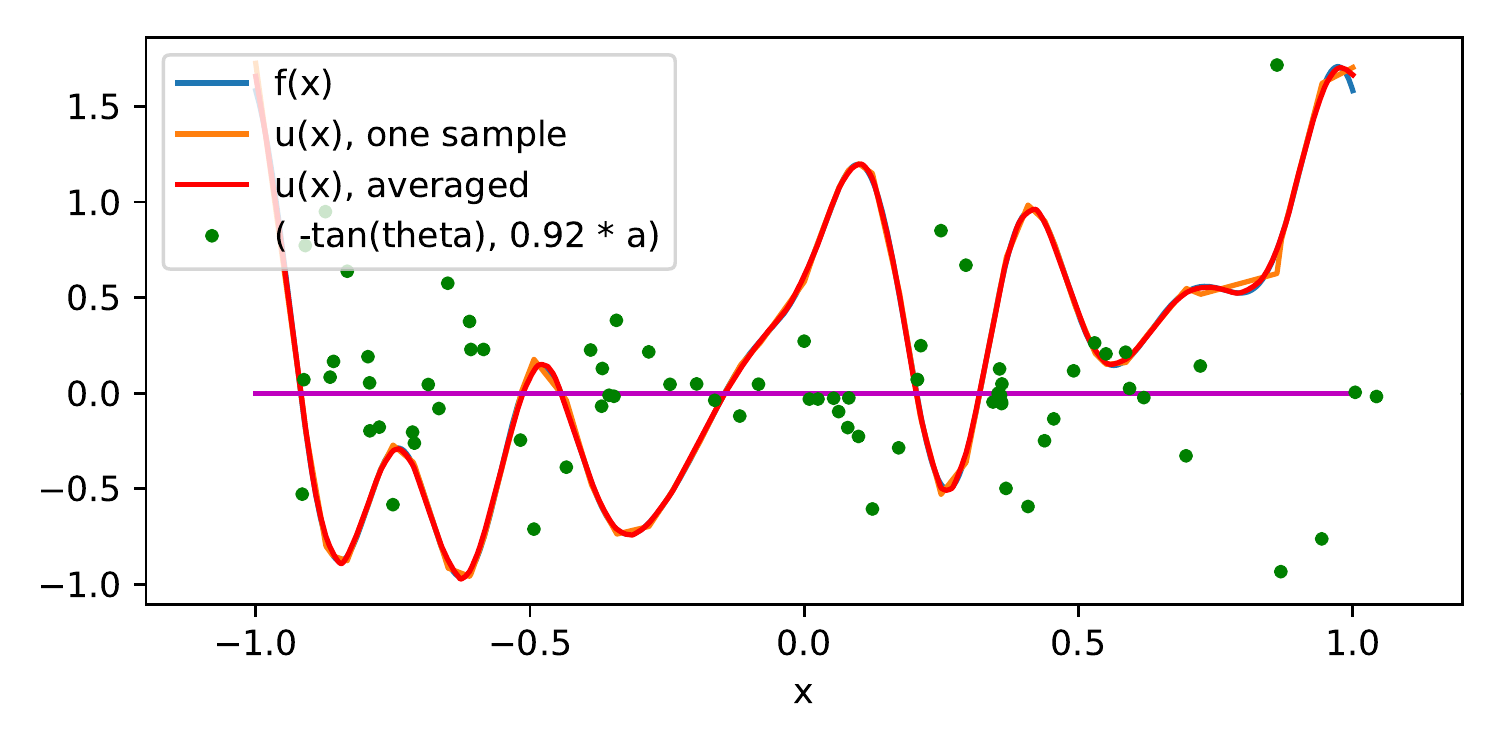}
        }
    \subfigure[$K=20,t=0$]{
        \includegraphics[width=0.46\textwidth]{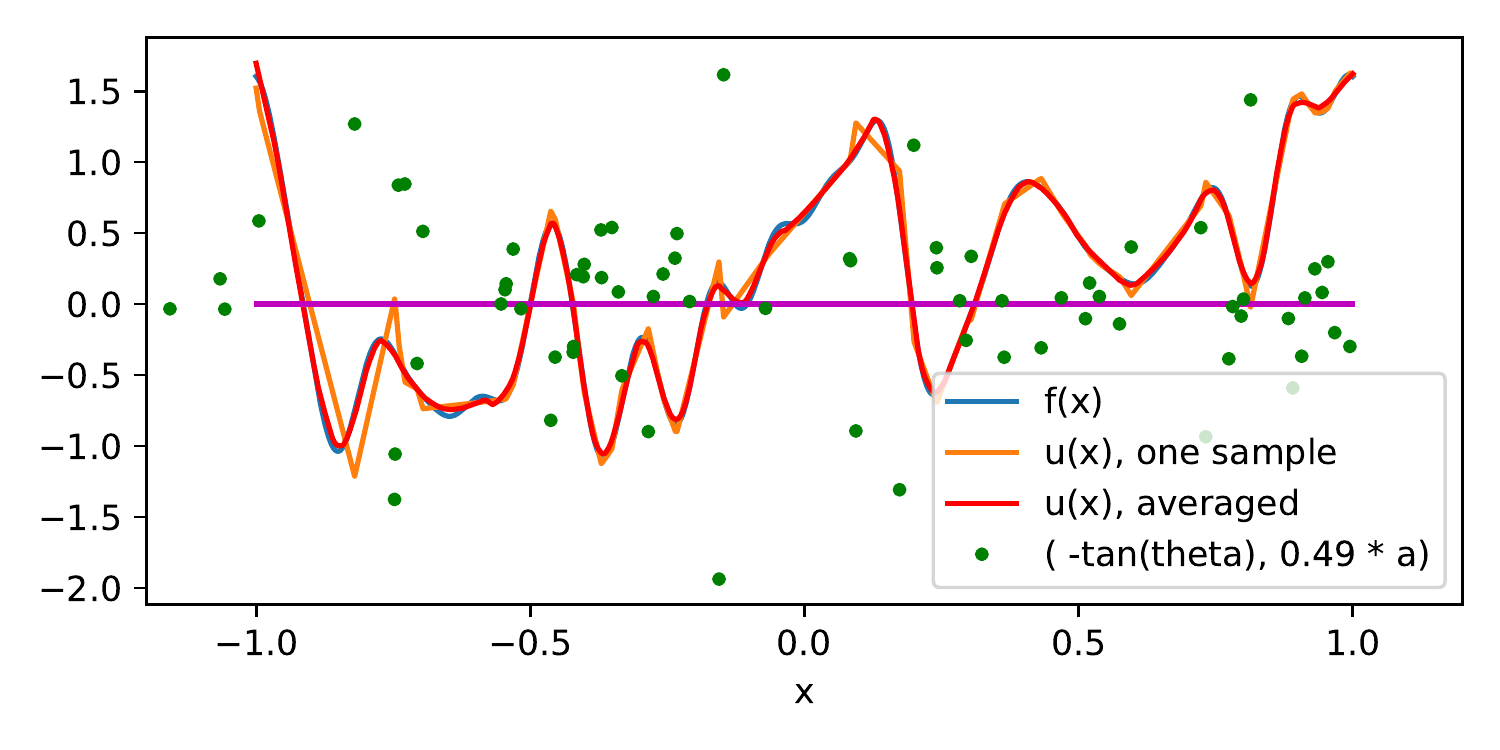}
        }
    \subfigure[$K=10,t=1$]{
        \includegraphics[width=0.46\textwidth]{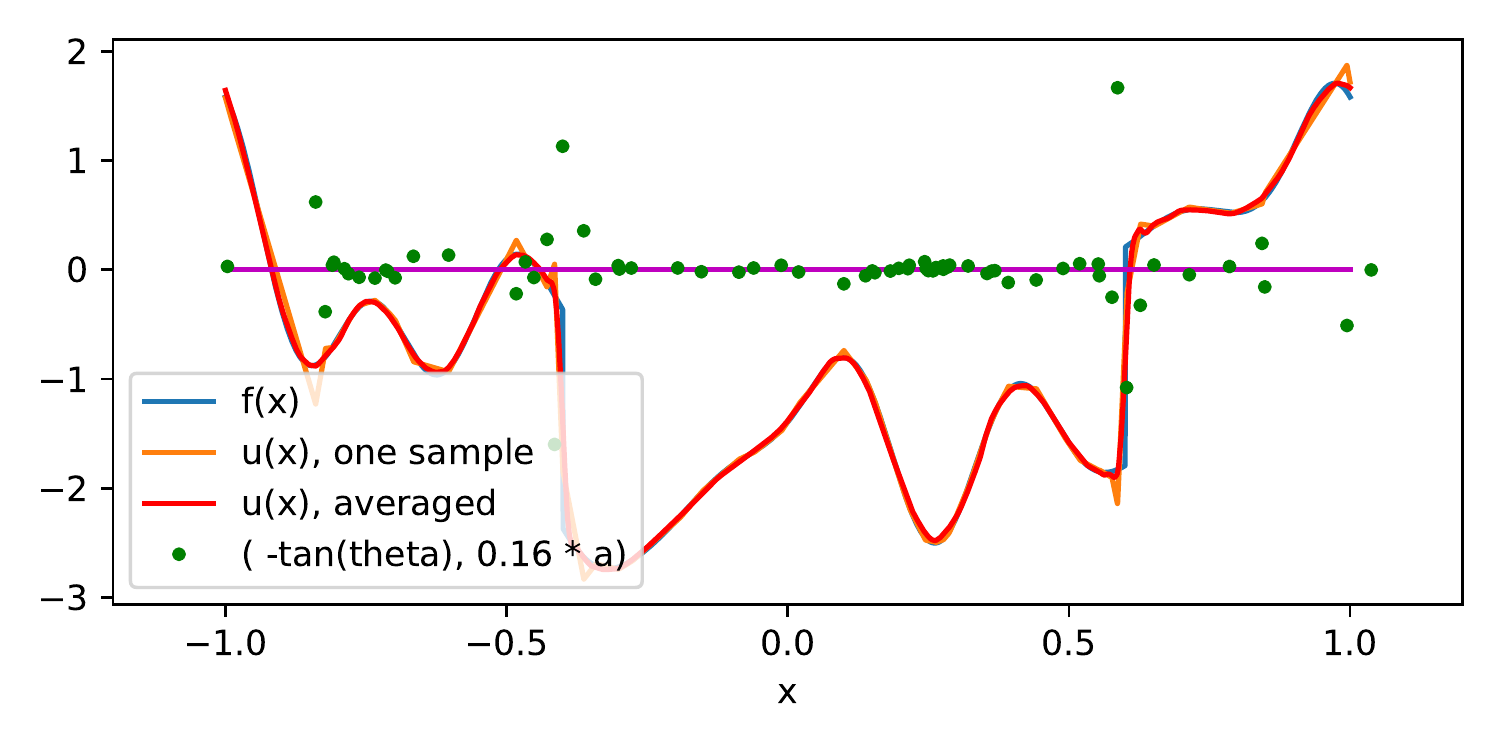}
        }
    \subfigure[$K=10,t=5$]{
        \includegraphics[width=0.46\textwidth]{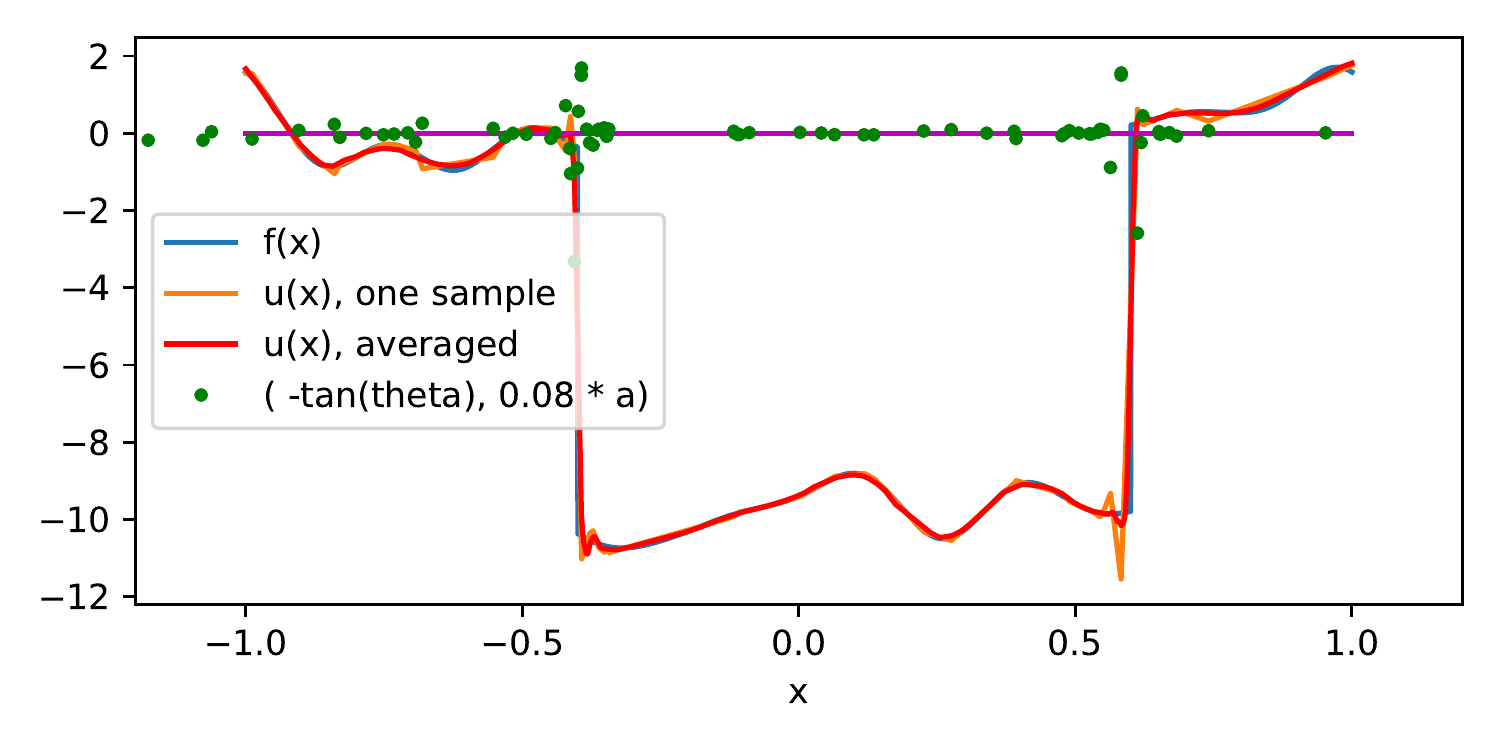}
        }
    \caption{Optimization in distribution space. Profile of the approximations given by Alg.~\ref{alg:A}.}
    \label{fig:1d_N100_profiles_minmodel}
\end{figure}


\textbf{Optimization in distribution space}
The convex combination coefficients $\alpha$ are initialized by normalizing a vector sampled randomly from the uniform distribution on $[0,1]$. In Fig.~\ref{fig:1d_N100_profiles_minmodel}, the profile of the approximations are given by Alg.~\ref{alg:A}, where both $u(x)$ coming from one sample and $u(x)$ averaged over 20 samples are given. In Fig.~\ref{fig:1d_N100_task02b}(a), the effect of $R$ (the number of terms to average in the Monte-Carlo integration in the Alg.\ref{alg:A}) is shown where the number of mixture basis $n=100$ is fixed. In Fig.~\ref{fig:1d_N100_task02b}(b), we add a result that uses the Gaussian basis instead of the triangle distribution basis shown in Fig.~\ref{fig:robustness}(a).

\begin{figure}[thb!]
    \centering
    \subfigure[The effect of $R$]{
        \includegraphics[width=0.47\textwidth]{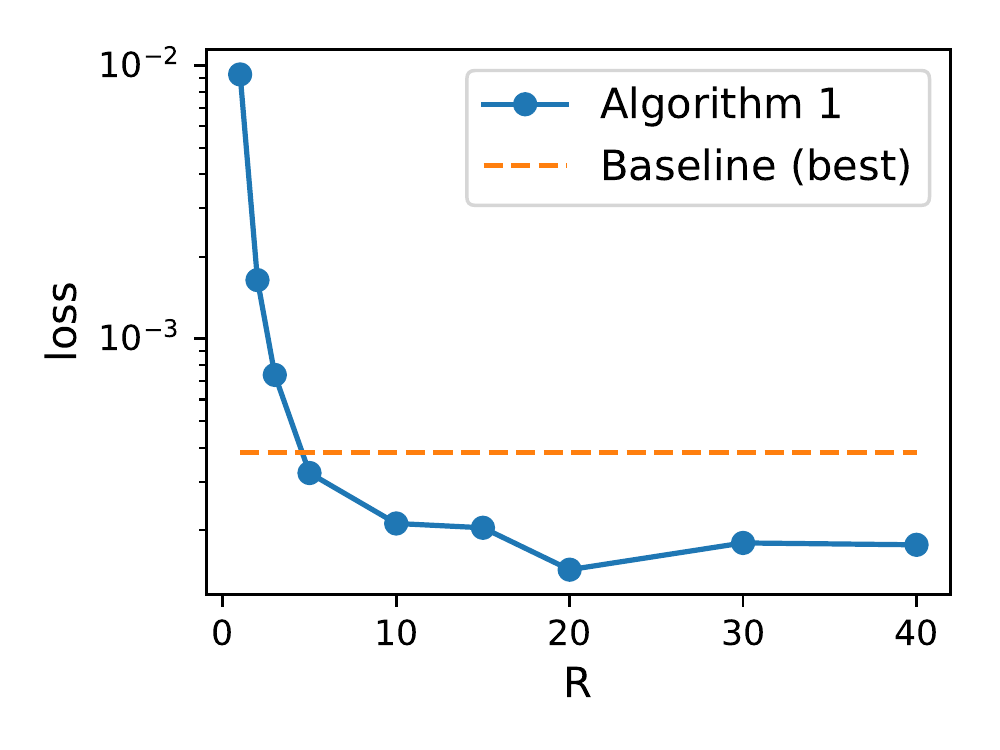}
        } 
    \subfigure[Gaussian basis ($R=20$)]{
        \includegraphics[width=0.46\textwidth]{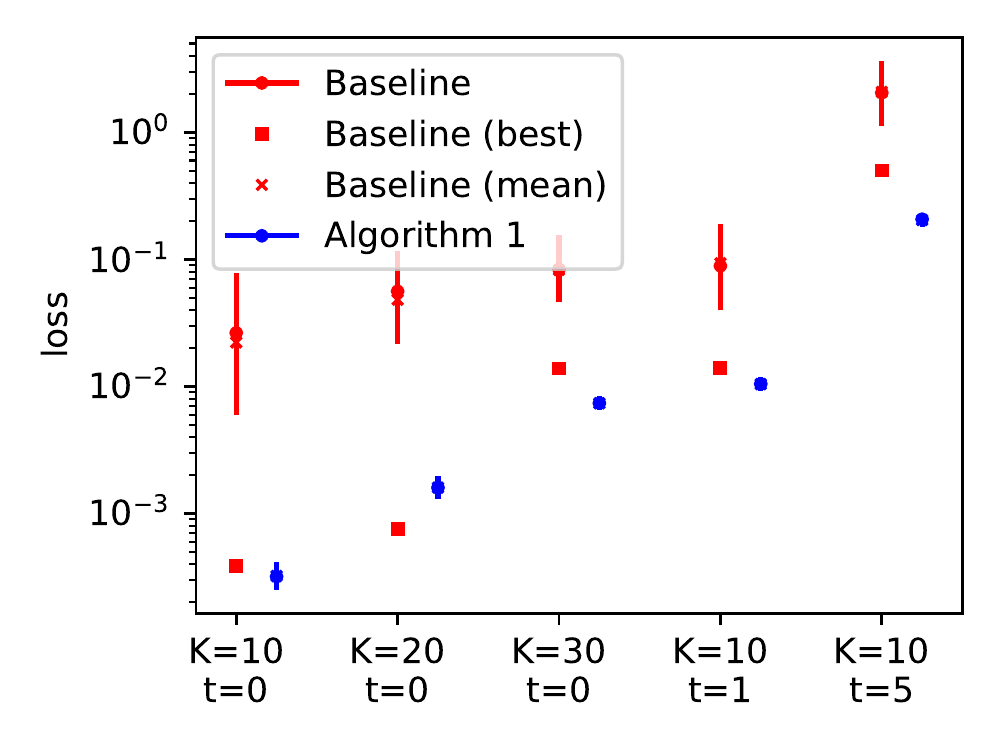}
        } 
    \caption{Optimization in distribution space. The number of mixture basis $n=100$ is fixed.}
    \label{fig:1d_N100_task02b}
\end{figure}

~
\newpage
\subsection{MNIST classification}

\textbf{Direct optimization}
In the tests of direct optimization in parameter space, the weights in first layer are drawn from a uniform distribution within $[-l, l]$ where
$l = \sqrt{6 / (784 + N)} s$ and $s$ is a scale. The other weights are initialized by the Glorot scheme \citep{Glorot2010Understanding}. Then we train the neural network by applying 30 epochs of Adam \citep{Kingma2014Adam} with learning rate 0.001. The scale $s$ is changed from $2^{-16}$ to $2^{16}$ and the resulting loss and accuracy for $N=100$ and $N=1000$ are given in Figure  \ref{fig:mnist_task01}.

\begin{figure}[thb!]
    \centering
    \subfigure[$N=100$]{
        \includegraphics[width=0.45\textwidth]{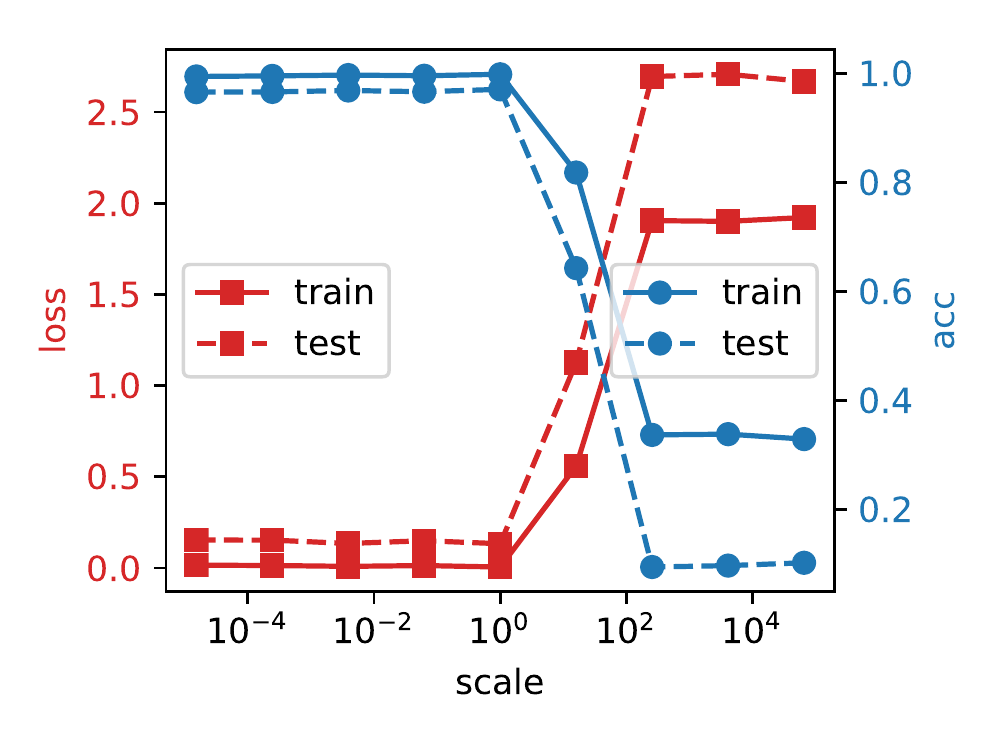}
        }
    \subfigure[$N=1000$]{
        \includegraphics[width=0.45\textwidth]{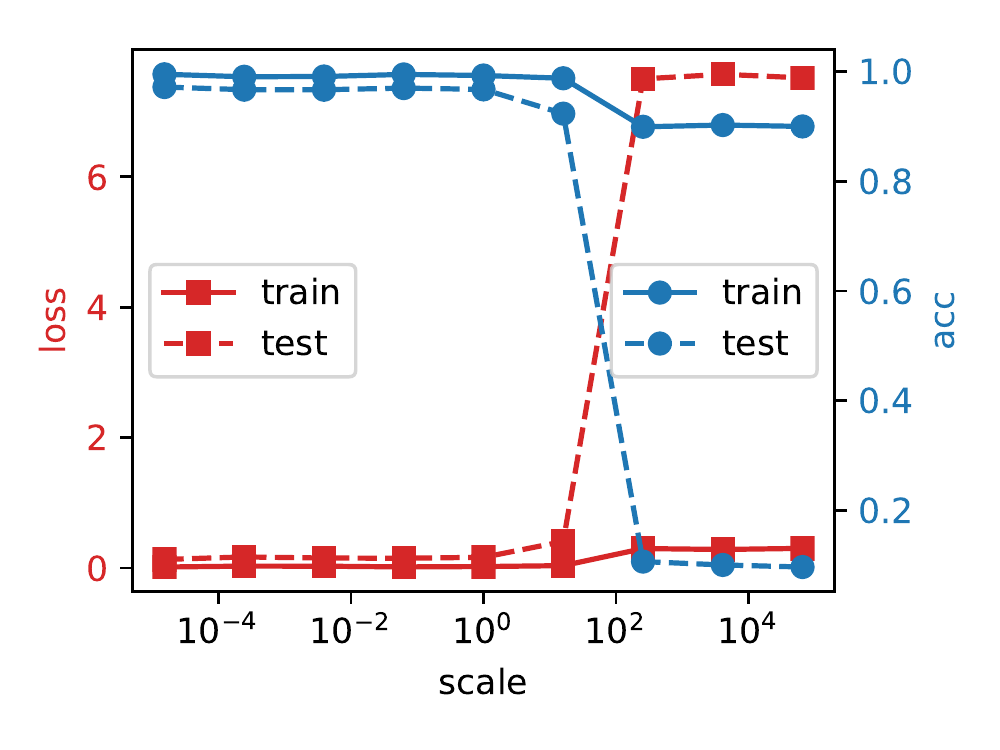}
        }
    \caption{Direct optimization in parameter space for MNIST classification.}
    \label{fig:mnist_task01}
\end{figure}

\textbf{Optimization in distribution space} In Fig.~\ref{fig:mnist_task03zi}(a-b), the effect of numerical parameter $R$ in the Alg.\ref{alg:A} are tested where the number of basis $n=65$ is fixed. The effect of $n$ for $N=1000$ is given in Fig.~\ref{fig:mnist_task03zi}(c) where $R=10$ is fixed.
In Fig.~\ref{fig:mnist_task03zi_example}, an example of the training result is given, where the initial convex combination coefficient $\alpha_i$ is uniform and after $k=300$ iterations, $\alpha_i$ is zero if $\lambda_i \not\in [4,12]$. Note that initializing $\alpha$ according to the logarithmic normal distribution results in similar results.

\begin{figure}[thb!]
    \centering
    \subfigure[Effect of $R$ ($N=100)$]{
        \includegraphics[width=0.3\textwidth]{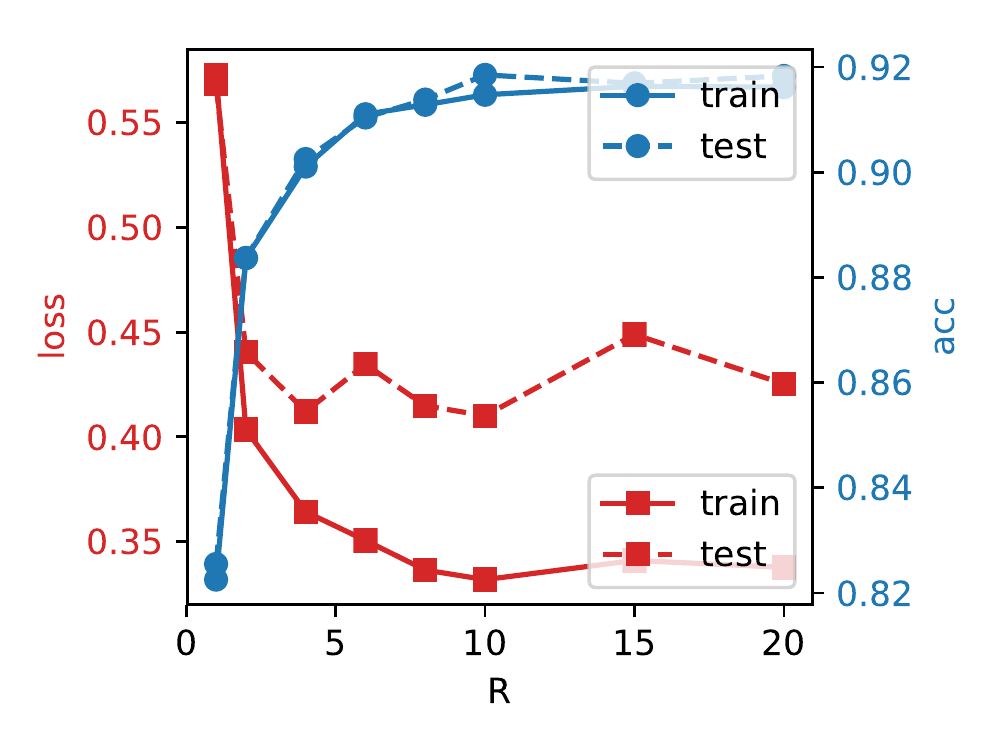}
        }
    \subfigure[Effect of $R$ ($N=1000)$]{
        \includegraphics[width=0.3\textwidth]{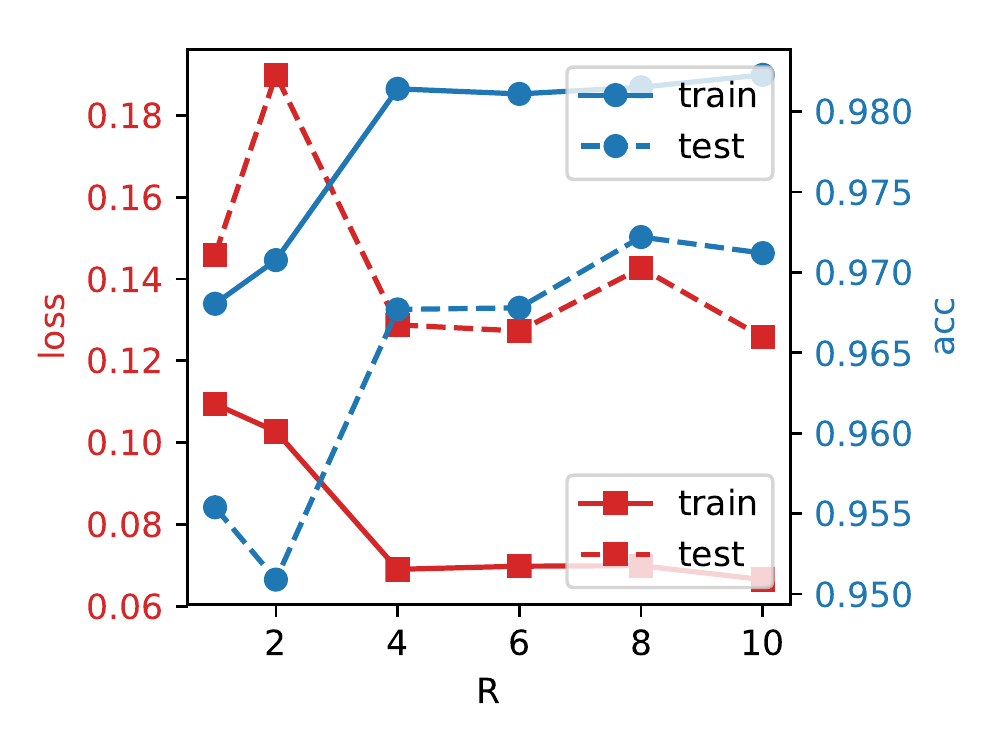}
        }
    \subfigure[Effect of $n$ ($N=1000)$]{
        \includegraphics[width=0.3\textwidth]{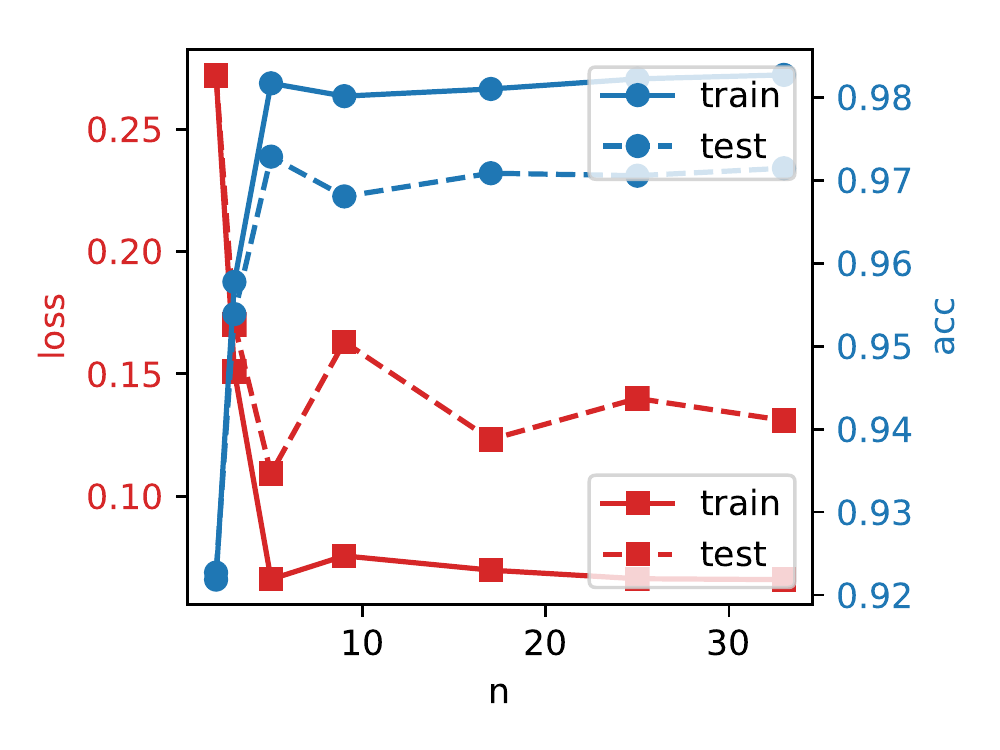}
        }
    \caption{Optimization in distribution space for MNIST classification.}
    \label{fig:mnist_task03zi}
\end{figure}

\begin{figure}[htp!]
    \centering
    \includegraphics[width=0.92\textwidth]{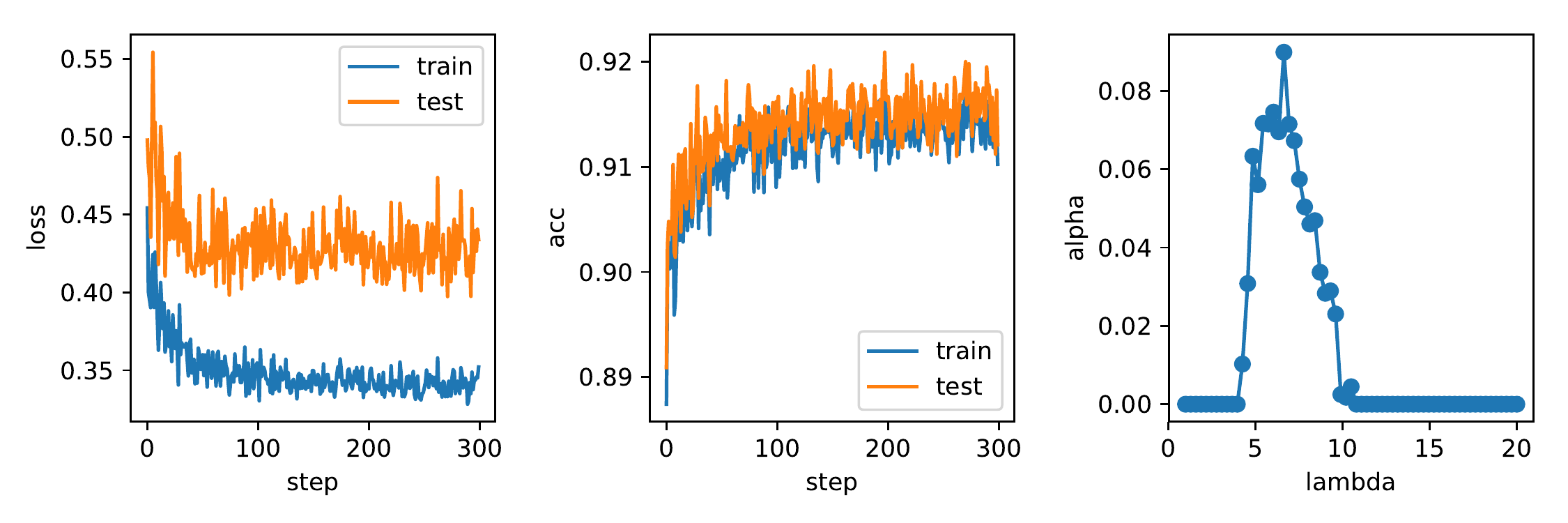}
    \caption{An example of the optimization in distribution space for MNIST classification ($N=100, n=65, R=10$).}
    \label{fig:mnist_task03zi_example}
\end{figure}

\subsection{MNIST classification with multi-layer neural networks}
\label{sec:mnist_cnn}

The multi-layer neural network consists of two convolutional layers (32 and 16 channels respectively) with $5\times5$ filters, each followed by leaky ReLU (with parameter 0.05) and $2\times2$ max pooling, followed by a fully connected layer and a batch normalization layer. For the MNIST training dataset, we only use the first 5000 samples (about 10\% of the MNIST dataset) in both the direct optimization in parameter space and the optimization in distribution space.

\textbf{Direct optimization} In the tests of direct optimization in parameter space, the other weights are initialized by the Glorot scheme \citep{Glorot2010Understanding} except that the bias in convolutional layers is drawn from a Gaussian distribution $\mathcal{N}(0,\sigma_0^2)$. Then we train the neural network by applying 30 epochs of Adam \citep{Kingma2014Adam} with learning rate 0.001. The scale $\sigma_0$ is changed from $10^{-3}$ to $10^{3}$ and the resulting loss and accuracy are given in Fig.~\ref{fig:mnist_cnn_direct}.

\begin{figure}[thb!]
    \centering
    \includegraphics[width=0.7\textwidth]{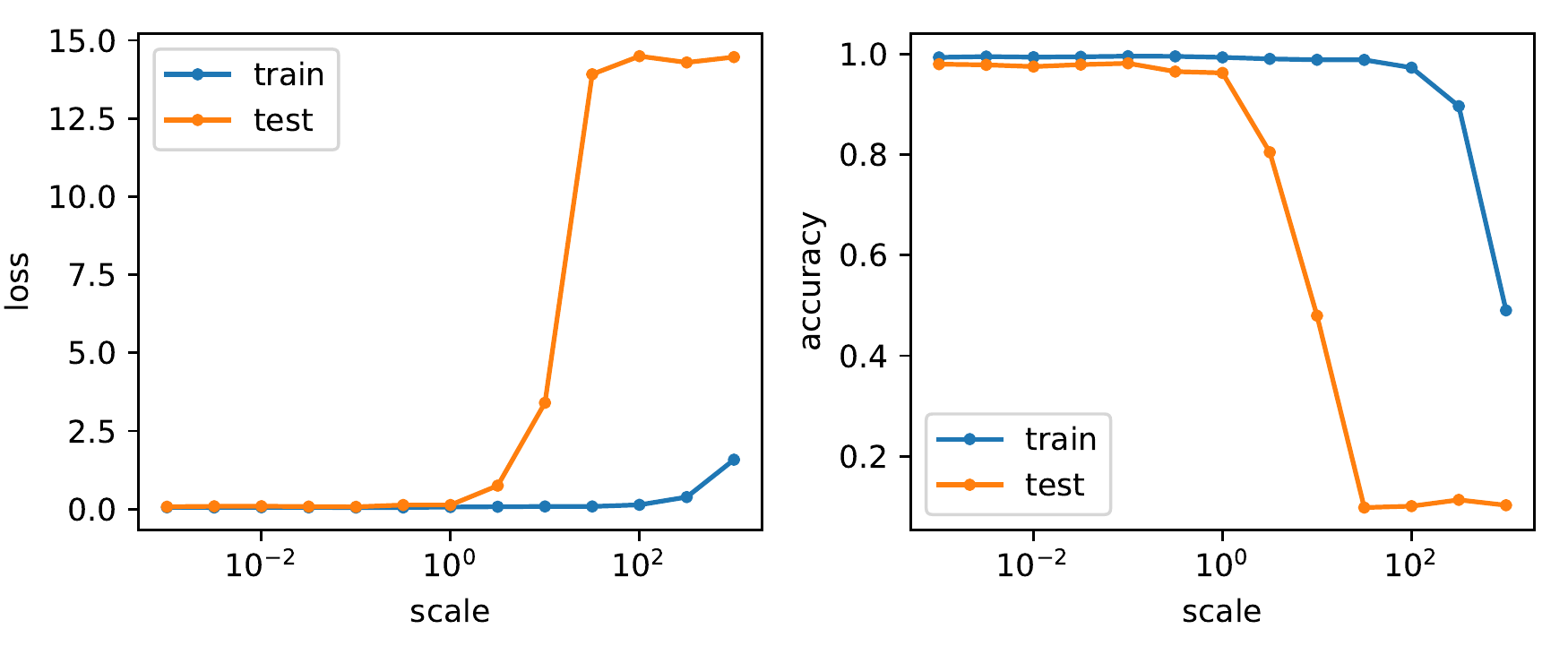}
    \caption{Direct optimization in parameter space for MNIST classification (CNN).}
    \label{fig:mnist_cnn_direct}
\end{figure}

\textbf{The distribution basis} For the hyper-network, we adopt the architecture designed in \citep{Deutsch2018Generating}, which contains an extractor and some weight generators, where the target network consists of two convolutional layers, followed by two fully connected layers, each uses the leaky-ReLU activation. Here we use a narrower configuration, which takes the input vector $z$ to be 150 dimensional, drawn from a uniform distribution. The total number of parameters in this hyper-network is 268090. The hyper-network is trained on fashion-MNIST dataset \citep{Xiao2017Fashion} by minimizing the following loss function $\tilde L(\Phi)$ with hyper-parameter $\lambda$:
\begin{equation}
    \tilde L \left(\Phi | p_{\text {noise }}, p_{\text {data }}\right)
    =\lambda L_{\text {accuracy }}\left(\Phi | p_{\text {noise }}, p_{\text {data }}\right)
    +L_{\text {diversity }}\left(\Phi | p_{\text {noise }}\right),
\end{equation}
where $L_{\text {accuracy }}$ and $L_{\text {diversity }}$ are the accuracy loss and diversity loss (quantified by the negative of the entropy of the outputs) respectively. Note that the hyper-network can generated the weights for both convolutional layers and fully connected layers, while we use the weights in convolutional layers as our $w_{\text{CNN}}$, and for the weights in fully connected layers we only use their standard deviation (which affect the scale of the output of the target network) to determine the parameter $\sigma$ in initialization of $w_{FC}$. We enumerate the distributions $\{\tilde \phi_i(w_{\text{CNN}})\}$ as a collection of hyper-networks with different $\lambda \in \text{logspace}(2,3,17)$ and different training steps $k_{\text{hyper}} \in \{5000,10000,20000\}$. Once the $w_{\text{CNN}}$ is generated, we fix it and turn to determine the $w_{FC}$ which is initialized (element-wise) by independent identically distributed Gaussian $\mathcal{N}(0,\sigma^2)$, and then updated by one epoch Adam-optimizer with learning rate $10^{-3}$ on 5000 MNIST training samples.

\textbf{Optimization in distribution space} 
In Fig.~\ref{fig:mnist_cnn_example}, an example of the training result with $R=50$ and $n=51$ is given, where the initial convex combination coefficient $\alpha_i$ is uniform and after $k=300$ iterations, $\alpha_i$ is zero if its corresponding $\lambda$ is small than 200.

\begin{figure}[thb!]
    \center
    \includegraphics[width=0.92\textwidth]{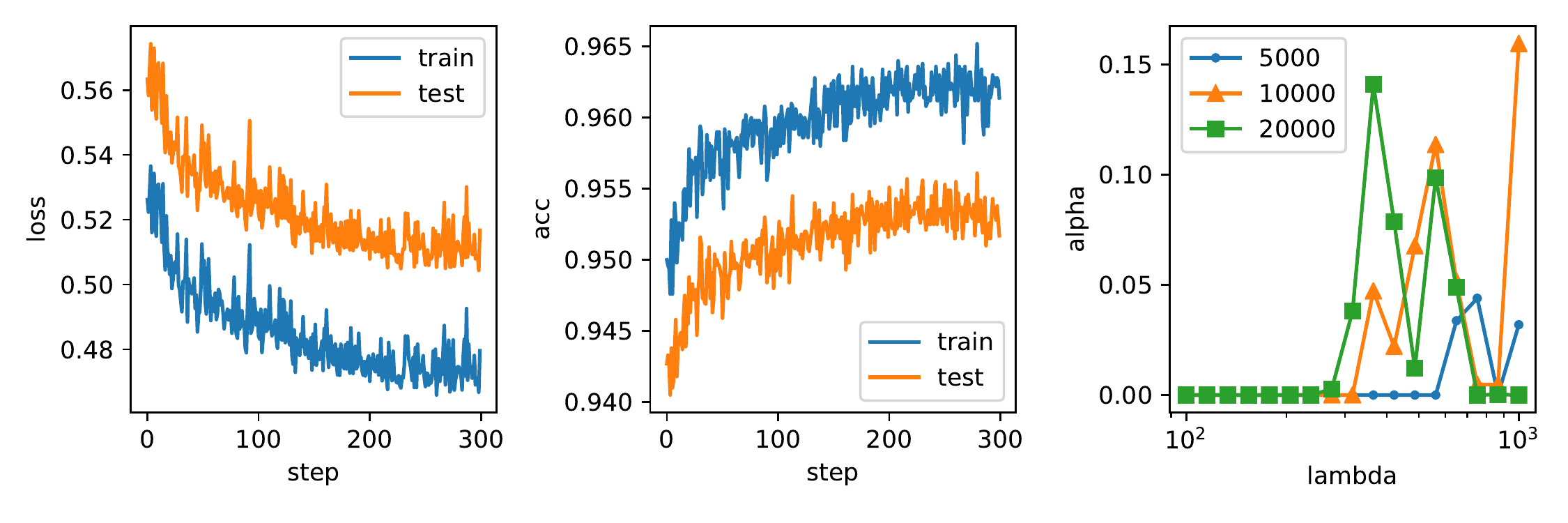}
    \caption{
        An example of the optimization in distribution space for MNIST classification with multi-layer neural networks ($n=51,R=50$).
    }
    \label{fig:mnist_cnn_example}
\end{figure}